\documentclass{article} 

\usepackage{nips15submit_e,times}

\usepackage[utf8]{inputenc}
\usepackage[colorlinks=true,citecolor=blue]{hyperref}
\usepackage{url}
\usepackage{wrapfig}

\usepackage[ruled]{algorithm2e}

\let\chapter\section
\usepackage{times}
\usepackage{graphicx} 
\usepackage{subfigure} 
\usepackage[sort&compress,square,numbers]{natbib}
\usepackage{algorithmic}
\usepackage{amsmath,amssymb,enumerate,comment}

\usepackage{enumitem}
\usepackage{bm}
\usepackage{amsthm}
\usepackage{bbm}
\usepackage{mathtools}

\newtheorem{lemma}{Lemma}
\newtheorem{theorem}{Theorem}
\newtheorem{corollary}{Corollary}

\newcommand{\Schedule}{\textsc{ScheduleUpdate}}
\newcommand{\sag}{\textsc{Sag}}
\newcommand{\sgd}{\textsc{Sgd}}
\newcommand{\csgd}{\textsc{CSgd}}
\newcommand{\dsgd}{\textsc{DSgd}}
\newcommand{\saga}{\textsc{Saga}}
\newcommand{\stgd}{\textsc{S2gd}}
\newcommand{\svrg}{\textsc{Svrg}}
\newcommand{\gd}{\textsc{GradientDescent}}
\newcommand{\svag}{\textsc{Hsag}}
\newcommand{\ndiv}{\hspace{-4pt}\not|\hspace{2pt}}

\newcommand{\reals}{\mathbb{R}}

\newcommand{\nlsum}{\sum\nolimits}
\newcommand{\E}{\mathbb{E}}

{\unskip\nobreak\hskip 1em plus 1fil\nobreak
\parfillskip=0pt%
\endtrivlist}

\numberwithin{equation}{section}

\title{On Variance Reduction in Stochastic Gradient Descent and its Asynchronous Variants}

\author{
Sashank J. Reddi \\
Carnegie Mellon University \\
\texttt{sjakkamr@cs.cmu.edu} \\
\AND
Ahmed Hefny \\
Carnegie Mellon University \\
\texttt{ahefny@cs.cmu.edu} \\
\And
Suvrit Sra \\
Massachusetts Institute of Technology \\
\texttt{suvrit@mit.edu} \\
\And
Barnab\'{a}s P\'{o}czos \\
Carnegie Mellon University \\
\texttt{bapoczos@cs.cmu.edu} \\
\And
Alex Smola \\
Carnegie Mellon University \\
\texttt{alex@smola.org} \\
}

\nipsfinalcopy 

\begin{document}

\maketitle

\begin{abstract}
  We study optimization algorithms based on variance reduction for stochastic gradient descent (SGD). Remarkable recent progress has been made in this direction through development of algorithms like SAG, SVRG, SAGA. These algorithms have been shown to outperform SGD, both theoretically and empirically. However, asynchronous versions of these algorithms---a crucial requirement for modern large-scale applications---have not been studied. We bridge this gap by presenting a unifying framework for many variance reduction techniques. Subsequently, we propose an asynchronous algorithm grounded in our framework, and prove its fast convergence. An important consequence of our general approach is that it yields asynchronous versions of variance reduction algorithms such as SVRG and SAGA as a byproduct. Our method achieves near linear speedup in sparse settings common to machine learning. We demonstrate the empirical performance of our method through a concrete realization of asynchronous SVRG.
\end{abstract}

\section{Introduction}
There has been a steep rise in recent work~\citep{Defazio14,Johnson13,Schmidt13,Xiao14,Konecny15,Konecny2013,sdca,defazio2014finito,mert15} on ``variance reduced'' stochastic gradient algorithms for convex problems of the \emph{finite-sum} form:
\begin{equation}
  \label{eq:1}
  \min_{x\in \reals^d}\ f(x) := \tfrac{1}{n}\nlsum_{i=1}^n f_i(x).
\end{equation}
Under strong convexity assumptions, such variance reduced (VR) stochastic algorithms attain better convergence rates (in expectation) than stochastic gradient descent (SGD)~\citep{RobMon51,nemirov09}, both in theory and practice.\footnote{Though we should note that SGD also applies to the harder stochastic optimization problem $\min\ F(x)=\E[f(x;\xi)]$, which need not be a finite-sum.} The key property of these VR algorithms is that by exploiting problem structure and by making suitable space-time tradeoffs, they reduce the variance incurred due to stochastic gradients. This variance reduction has powerful consequences: it helps VR stochastic methods attain linear convergence rates, and thereby circumvents slowdowns that usually hit SGD. 

Although these advances have great value in general, for large-scale problems we still require parallel or distributed processing. And in this setting, asynchronous variants of SGD remain indispensable~\citep{Recht11,zinkevich10,agDuc11,shamir2014,dekel2012,li2014}. Therefore, a key question is how to extend the synchronous finite-sum VR algorithms to asynchronous parallel and distributed settings.

We answer one part of this question by developing new asynchronous parallel stochastic gradient methods that provably converge at a linear rate for smooth strongly convex finite-sum problems. Our methods are inspired by the influential \svrg~\citep{Johnson13}, \stgd~\citep{Konecny2013}, \sag~\citep{Schmidt13} and \saga~\citep{Defazio14} family of algorithms. We list our contributions more precisely below.

\textbf{Contributions.} Our paper makes two core contributions: (i) a formal general framework for variance reduced stochastic methods based on discussions in \cite{Defazio14}; and (ii) asynchronous parallel VR algorithms within this framework. Our general framework presents a formal unifying view of several VR methods (e.g., it includes SAGA and SVRG as special cases) while expressing key algorithmic and practical tradeoffs concisely. Thus, it yields a broader understanding of VR methods, which helps us obtain \emph{asynchronous parallel} variants of VR methods. Under sparse-data settings common to machine learning problems, our parallel algorithms attain speedups that scale near linearly with the number of processors. 

As a concrete illustration, we present a specialization to an asynchronous \svrg-like method. We compare this specialization with non-variance reduced asynchronous SGD methods, and observe strong empirical speedups that agree with the theory.

\textbf{Related work.} As already mentioned, our work is closest to (and generalizes) \sag~\citep{Schmidt13}, \saga~\citep{Defazio14}, \svrg~\citep{Johnson13} and \stgd~\citep{Konecny2013}, which are primal methods. Also closely related are dual methods such as \textsc{sdca}~\citep{sdca} and  Finito~\citep{defazio2014finito}, and in its convex incarnation \textsc{Miso}~\citep{mairal2013}; a more precise relation between these dual methods and VR stochastic methods is described in Defazio's thesis~\citep{defaziothesis}. By their algorithmic structure, these VR methods trace back to classical non-stochastic incremental gradient algorithms~\citep{bertsekas2011}, but by now it is well-recognized that randomization helps obtain much sharper convergence results (in expectation).  Proximal~\citep{Xiao14} and accelerated VR methods have also been proposed~\citep{asdca,nitanda14}; we leave a study of such variants of our framework as future work. Finally, there is recent work on lower-bounds for finite-sum problems~\citep{agarwal2014}.

Within  asynchronous SGD algorithms, both parallel~\citep{Recht11} and distributed~\citep{nedic2001distributed,agDuc11} variants are known. In this paper, we focus our attention on the parallel setting. A different line of methods is that of (primal) coordinate descent methods, and their parallel and distributed variants~\citep{nesterov2012,richtarik2014,Liu14,liu15,reddi2014large}. Our asynchronous methods share some structural assumptions with these methods. Finally, the recent work~\citep{Konecny15} generalizes S2GD to the mini-batch setting, thereby also permitting parallel processing, albeit with more synchronization and allowing only small mini-batches.



\vspace*{-4pt}
\section{A General Framework for VR Stochastic Methods}
\label{sec:gen-framework}
\vspace*{-6pt}
We focus on instances of~\eqref{eq:1} where the cost function $f(x)$ has an $L$-Lipschitz gradient, so that $\|\nabla f(x)-\nabla f(y)\| \le L\|x-y\|$, and it is $\lambda$-strongly convex, i.e., for all $x, y \in \reals^d$, 
\begin{equation}
  \label{eq:2}
  f(x) \ge f(y) + \langle \nabla f(y), x - y \rangle + \tfrac\lambda2\|x-y\|^2.
\end{equation}
While our analysis focuses on strongly convex functions, we can extend it to just smooth convex functions along the lines of \cite{Defazio14,Xiao14}.


Inspired by the discussion on a general view of variance reduced techniques in \cite{Defazio14}, we now describe a formal general framework for variance reduction in stochastic gradient descent. We denote the collection $\{f_i\}_{i=1}^n$ of functions that make up $f$ in ~\eqref{eq:1} by $\cal F$. For our algorithm, we maintain an additional parameter $\alpha_i^t \in \mathbb{R}^d$ for each $f_i \in \mathcal{F}$. We use $A^t$ to denote  $\{\alpha_i^t\}_{i=1}^n$. The general iterative framework for updating the parameters is presented as  Algorithm~\ref{alg:gsvrg}. Observe that the algorithm is still abstract, since it does not specify the subroutine $\Schedule$. This subroutine determines the crucial update mechanism of $\{\alpha_i^t\}$ (and thereby of $A^t$). As we will see different schedules give rise to different fast first-order methods proposed in the literature. The part of the update based on $A^t$ is the key for these approaches and is responsible for variance reduction.

\begin{algorithm}[t]
\SetAlgoLined
\KwData{$x^0 \in \mathbb{R}^d, \alpha_i^0 = x^0$ $\ \forall i \in [n] \triangleq \{1, \dots, n\}$,  step size $\eta > 0$}
Randomly pick a  $I_T = \{i_0,\dots, i_T\}$ where $i_t \in \{1,\dots,n\} \ \forall$ $t \in \{0, \dots, T\}$ \;
\For {$t = 0$ \textbf{to} $T$} {
	Update iterate as
  	$x^{t+1} \leftarrow x^t - \eta \left(\nabla f_{i_t}(x^t) - \nabla f_{i_t}(\alpha_{i_t}^t) + \frac{1}{n} \sum_i \nabla f_{i}(\alpha_{i}^t)\right)$ \;
  	$A^{t+1} = \Schedule(\{x^i\}_{i=0}^{t+1}, A^t, t, I_T)$ \;
}
\textbf{return} $x^T$
\caption{\textsc{Generic Stochastic Variance Reduction Algorithm}}
\label{alg:gsvrg}
\end{algorithm}

Next, we provide different instantiations of the framework and construct a new algorithm derived from it. In particular, we consider incremental methods $\sag$ \cite{Schmidt13}, $\svrg$ \cite{Johnson13} and $\saga$ \cite{Defazio14}, and classic gradient descent $\gd$ for demonstrating our framework.

Figure~\ref{fig:all-sched} shows the schedules for the aforementioned algorithms. In case of $\svrg$, $\Schedule$ is triggered every $m$ iterations (here $m$ denotes  precisely the number of inner iterations used in \cite{Johnson13}); so $A^t$ remains unchanged for the $m$ iterations and all $\alpha_i^t$ are updated to the current iterate at the $m^{\text{th}}$ iteration. For $\saga$, unlike $\svrg$, $A^t$ changes at the $t^{th}$ iteration for all $t \in [T]$. This change is only to a single element of $A^t$, and is determined by the index $i_t$ (the function chosen at iteration $t$). The update of $\sag$ is similar to $\saga$ insofar that only one of the $\alpha_i$ is updated at each iteration. However, the update for $A^{t+1}$ is based on $i_{t+1}$ rather than $i_t$. This results in a biased estimate of the gradient, unlike $\svrg$ and $\saga$. Finally, the schedule for gradient descent is similar to $\sag$, except that all the $\alpha_i$'s are updated at each iteration. Due to the full update we end up with the exact gradient at each iteration. This discussion highlights how the scheduler determines the resulting gradient method.

To motivate the design of another schedule, let us consider the computational and storage costs of each of these algorithms. For $\svrg$, since we update $A^t$ after every $m$ iterations, it is enough to store a full gradient, and hence, the storage cost is $O(d)$. However, the running time is $O(d)$ at each iteration and $O(nd)$ at the end of each epoch (for calculating the full gradient at the end of each epoch). In contrast, both $\sag$ and $\saga$ have high storage costs of $O(nd)$ and running time of $O(d)$ per iteration. Finally, $\gd$ has low storage cost since it needs to store the gradient at $O(d)$ cost, but very high computational costs of $O(nd)$ at \emph{each} iteration. 

\svrg \ has an additional computation overhead at the end of each epoch due to calculation of the whole gradient. This is avoided in $\sag$ and $\saga$ at the cost of additional storage. When $m$ is very large, the additional computational overhead of $\svrg$ amortized over all the iterations is small. However, as we will later see, this comes at the expense of slower convergence to the optimal solution. The tradeoffs between the epoch size $m$, additional storage,  frequency of updates, and the convergence to the optimal solution are still not completely resolved.

\RestyleAlgo{boxed}

\begin{figure}[h]
\begin{minipage}[t]{6.85cm}
  \vspace{0pt}  
  \begin{algorithm}[H]
    \textbf{SVRG:}$\Schedule(\{x^i\}_{i=0}^{t+1}, A^t, t, I_T)$
    \For {i = 1 to n} {
	    $\alpha_i^{t+1} = \mathbbm{1}(m \mid t) x^t + \mathbbm{1}(m \ndiv t) \alpha_i^{t}$ \;
    }
    \textbf{return} $A^{t+1}$
  \end{algorithm}
\end{minipage}%
\hspace{0.15cm}
\begin{minipage}[t]{6.85cm}
  \vspace{0pt}
  \begin{algorithm}[H]
      \textbf{SAGA:}$\Schedule(\{x^i\}_{i=0}^{t+1}, A^t, t, I_T)$
      \For {i = 1 to n} {
  	    $\alpha_i^{t+1} = \mathbbm{1}(i_t = i) x^t + \mathbbm{1}(i_t \neq i) \alpha_i^{t}$ \;
      }
      \textbf{return} $A^{t+1}$
    \end{algorithm}
\end{minipage}

\begin{minipage}[t]{6.85cm}
  \vspace{0pt}  
  \begin{algorithm}[H]
    \textbf{SAG:}$\Schedule(\{x^i\}_{i=0}^{t+1}, A^t, t, I_T)$
    \For {i = 1 to n} {
      $\alpha_i^{t+1} = \mathbbm{1}(i_{t+1} = i) x^{t+1} +  \mathbbm{1}(i_{t+1} \neq i) \alpha_i^{t}$ \;
    }
    \textbf{return} $A^{t+1}$
  \end{algorithm}
\end{minipage}%
\hspace{0.15cm}
\begin{minipage}[t]{6.85cm}
  \vspace{0pt}  
  \begin{algorithm}[H]
    \textbf{GD:}$\Schedule(\{x^i\}_{i=0}^{t+1}, A^t, t, I_T)$ \
        \For {i = 1 to n} {
          $\alpha_i^{t+1} =  x^{t+1}$ \;
        }
    \textbf{return} $A^{t+1}$
  \end{algorithm}
\end{minipage}%

\caption{$\Schedule$ function for $\svrg$ (top left), $\saga$ (top right), $\sag$ (bottom left) and $\gd$ (bottom right). While $\svrg$ is epoch-based, rest of algorithms perform updates at each iteration. Here $a | b$ denotes that $a$ divides $b$.}
\label{fig:all-sched}
\end{figure}

A straightforward approach to design a new scheduler is to combine the schedules of the above algorithms. This allows us to tradeoff between the various aforementioned parameters of our interest. We call this schedule \emph{hybrid stochastic average gradient} ($\svag$). Here, we use the schedules of $\svrg$ and $\saga$ to develop $\svag$. However, in general, schedules of any of these algorithms can be combined to obtain a hybrid algorithm.  Consider some $S \subseteq [n]$, the indices that follow $\saga$ schedule. We assume that the rest of the indices follow an $\svrg$-like schedule with \emph{schedule frequency} $s_i$ for all $i \in \overline{S}\triangleq [n]\setminus S$. Figure~\ref{fig:svag-sched} shows the corresponding update schedule of $\svag$. If $S = [n]$ then $\svag$ is equivalent to $\saga$, while at the other extreme, for $S = \emptyset$ and $s_i = m$ for all $i \in [n]$, it corresponds to $\svrg$. $\svag$ exhibits interesting storage, computational and convergence trade-offs that depend on $S$. In general, while large cardinality of $S$ likely incurs high storage costs, the computational cost per iteration is relatively low. On the other hand, when cardinality of $S$ is small and $s_i$'s are large, storage costs are low but the convergence typically slows down.

\begin{figure}
  \centering
  \begin{minipage}[t]{8cm}
    \vspace{0pt}  
    \begin{algorithm}[H]
      \textbf{$\svag$:}$\Schedule(x^t, A^t, t, I_T)$\\
      \For {i = 1 to n} {
        $\alpha_i^{t+1} = \left\{
        	\begin{array}{ll}
        		\mathbbm{1}(i_{t} = i) x^t +  \mathbbm{1}(i_{t} \neq i) \alpha_i^{t}  & \mbox{if } i \in S \\
        		\mathbbm{1}(s_i \mid t) x^t + \mathbbm{1}(s_i \ndiv t) \alpha_i^{t} & \mbox{if } i \notin S
        	\end{array}
        \right.$
      }
      \textbf{return} $A^{t+1}$
    \end{algorithm}
  \end{minipage}
  \vspace*{-10pt}
  \label{fig:svag-sched}
  \caption{$\Schedule$ for $\svag$. This algorithm assumes access to some index set $S$ and the schedule frequency vector $s$. Recall that $a | b$ denotes $a$ divides $b$}
\end{figure}

Before concluding our discussion on the general framework, we would like to draw the reader's attention to the advantages of studying Algorithm~\ref{alg:gsvrg}. First, note that Algorithm~\ref{alg:gsvrg} provides a unifying framework for many incremental/stochastic gradient methods proposed in the literature. Second, and more importantly, it provides a generic platform for analyzing this class of algorithms. As we will see in Section~\ref{sec:asyn-framework}, this helps us develop and analyze asynchronous versions for different finite-sum algorithms under a common umbrella. Finally, it provides a mechanism to derive new algorithms by designing more sophisticated  schedules; as noted above, one such construction gives rise to $\svag$.

\subsection{Convergence Analysis}
In this section, we provide convergence analysis for Algorithm~\ref{alg:gsvrg} with $\svag$ schedules. As observed earlier, $\svrg$ and $\saga$ are special cases of this setup. Our analysis assumes unbiasedness of the gradient estimates at each iteration, so it does not encompass $\sag$. For ease of exposition, we assume that all $s_i = m$ for all $i \in [n]$. Since $\svag$ is epoch-based, our analysis focuses on the iterates obtained after each epoch. Similar to \cite{Johnson13} (see Option II of \svrg \ in \cite{Johnson13}), our analysis will be for the case where the iterate at the end of $({k+1})^{\text{st}}$ epoch, $x^{km+m}$, is replaced with an element chosen randomly from $\{x^{km},\dots,x^{km+m-1}\}$ with probability $\{p_1, \cdots, p_m\}$.  For brevity, we use $\tilde{x}^k$ to denote the iterate chosen at the $k^{\text{th}}$ \emph{epoch}. We also need the following quantity for our analysis:
\begin{align*}
\tilde{G}_{k} \triangleq \frac{1}{n} \sum_{i \in S} \left(f_{i}(\alpha_i^{km}) - f_i(x^*) - \langle \nabla f_i(x^*), \alpha_i^{km} - x^* \rangle \right).
\end{align*}

\begin{theorem}
For any positive parameters $c, \beta, \kappa > 1$, step size $\eta$ and epoch size $m$, we define the following quantities:
\begin{align*}
\gamma &= \kappa \left[1 - \left(1 - \frac{1}{\kappa}\right)^m \right] \left(2c\eta (1 - L\eta(1+\beta)) - \frac{1}{n} - \frac{2c}{\kappa\lambda} \right) \\
\theta &= \max \left\lbrace \left[\frac{2c}{\gamma\lambda}\left(1 - \frac{1}{\kappa}\right)^m + \frac{2Lc\eta^2}{\gamma}\left(1 + \frac{1}{\beta}\right)  \kappa \left[1 - \left(1 - \frac{1}{\kappa}\right)^m \right]  \right], \left(1 - \frac{1}{\kappa}\right)^m  \right\rbrace.
\end{align*}
Suppose the probabilities $p_i \propto (1 - \frac{1}{\kappa})^{m-i}$, and that $c, \beta, \kappa$, step size $\eta$ and epoch size $m$ are chosen such that the following conditions are satisfied:
$$
\frac{1}{\kappa} + 2Lc\eta^2 \left(1 + \frac{1}{\beta}\right) \leq \frac{1}{n}, \ \gamma > 0, \ \theta < 1.
$$
\label{thm:t1}
Then, for iterates of Algorithm~\ref{alg:gsvrg} under the $\svag$ schedule, we have
\begin{align*}
& \mathbb{E}\Bigl[f(\tilde{x}^{k+1}) - f(x^*) + \frac{1}{\gamma} \tilde{G}_{k+1}\Bigr] \leq  \theta \ \mathbb{E}\Bigl[f(\tilde{x}^{k}) - f(x^*) + \frac{1}{\gamma} \tilde{G}_{k}\Bigr].
\end{align*}
\end{theorem}
As a corollary, we immediately obtain an expected linear rate of convergence for \svag.
\begin{corollary}
Note that $\tilde{G}_k \geq 0$ and therefore, under the conditions specified in Theorem~\ref{thm:t1} and with $\bar{\theta} = \theta \left(1 + 1/\gamma\right) < 1$ we have
\begin{align*}
\mathbb{E}\left[f(\tilde{x}^{k}) - f(x^*)\right] \leq  \bar{\theta}^k \ \left[f(x^{0}) - f(x^*) \right].
\end{align*}
\label{col:col1}
\end{corollary}
We emphasize that there exist values of the parameters for which the conditions in Theorem~\ref{thm:t1} and Corollary~\ref{col:col1} are easily satisfied. For instance, setting $\eta = 1/16(\lambda n + L)$, $\kappa = 4/\lambda \eta$, $\beta = (2\lambda n + L)/L$ and $c = 2/\eta n$, the conditions in Theorem~\ref{thm:t1} are satisfied for sufficiently large $m$. Additionally, in the high condition number regime of $L/\lambda = n$, we can obtain constant $\theta < 1$ (say $0.5$) with $m = O(n)$ epoch size (similar to \cite{Johnson13,Defazio14}). This leads to a computational complexity  of $O(n \log(1/\epsilon))$ for \svag \ to achieve  $\epsilon$ accuracy in the objective function as opposed to $O(n^2 \log(1/\epsilon))$ for batch gradient descent method. Please refer to the appendix for more details on the parameters in Theorem~\ref{thm:t1}.

\vspace*{-4pt}
\section{Asynchronous Stochastic Variance Reduction}
\label{sec:asyn-framework}
\vspace*{-4pt}
We are now ready to present asynchronous versions of the algorithms captured by our  general framework. We first describe our setup before delving into the details of these algorithms. Our model of computation is similar to the ones used in Hogwild!~\cite{Recht11} and AsySCD \cite{Liu14}. We assume a multicore architecture where each core makes stochastic gradient updates to a centrally stored vector $x$ in an asynchronous manner. There are four key components in our asynchronous algorithm; these are briefly described below.
\begin{enumerate}
\setlength{\itemsep}{0pt}
\item {\bf Read}: Read the iterate $x$ and compute the gradient $\nabla f_{i_t}(x)$ for a randomly chosen $i_t$.
\item {\bf Read schedule iterate}: Read the schedule iterate $A$ and compute the gradients required for update in Algorithm~\ref{alg:gsvrg}.
\item {\bf Update}: Update the iterate $x$ with the computed incremental update in Algorithm~\ref{alg:gsvrg}.
\item {\bf Schedule Update}: Run a scheduler update for updating $A$.
\end{enumerate}

Each processor repeatedly runs these procedures concurrently, without any synchronization. Hence, $x$ may change in between Step 1 and Step 3. Similarly, $A$ may change in between Steps 2 and 4. In fact, the states of iterates $x$ and $A$ can correspond to different time-stamps. We maintain a global counter $t$ to track the number of updates successfully executed. We use $D(t) \in [t]$ and $D'(t) \in [t]$ to denote the particular $x$-iterate and $A$-iterate used for evaluating the update at the $t^{\text{th}}$ iteration. We assume that the delay in between the time of evaluation and updating is bounded by a non-negative integer $\tau$, i.e., $t - D(t) \leq \tau$ and $t - D'(t) \leq \tau$. The bound on the staleness captures the degree of parallelism in the method: such parameters are typical in asynchronous systems (see e.g.,~\cite{BerTsi89,Liu14}). Furthermore, we also assume that the system is synchronized after every epoch i.e., $D(t) \geq km$ for $t \geq km$. We would like to emphasize that the assumption is not strong since such a synchronization needs to be done only once per epoch.

For the purpose of our analysis, we assume a consistent read model. In particular, our analysis assumes that the vector $x$ used for evaluation of gradients is a valid iterate that existed at some point in time. Such an assumption typically amounts to using locks in practice. This problem can be avoided by using random coordinate updates as in \cite{Recht11} (see Section 4 of \cite{Recht11}) but such a procedure is computationally wasteful in practice. We leave the analysis of inconsistent read model as future work. Nonetheless, we report results for both locked and lock-free implementations (see Section~\ref{sec:expt}).

\vspace*{-4pt}
\subsection{Convergence Analysis}
\label{sec:strongly-convex-conv}
\vspace*{-4pt}
The key ingredients to the success of asynchronous algorithms for multicore stochastic gradient descent are sparsity and ``disjointness'' of the data matrix~\citep{Recht11}. More formally, suppose $f_i$ only depends on $x_{e_i}$ where $e_i \subseteq [d]$ i.e., $f_i$ acts only on the components of $x$ indexed by the set $e_i$. Let $\|x\|^2_i$ denote $\sum_{j \in e_i} \|x_{j}\|^2$; then, the convergence depends on $\Delta$, the smallest constant such that $\mathbb{E}_i[\|x\|_{i}^2] \leq \Delta \|x\|^2$. Intuitively, $\Delta$ denotes the average frequency with which a feature appears in the data matrix. We are interested in situations where $\Delta \ll 1$. As a warm up, let us first discuss convergence analysis for asynchronous $\svrg$. The general case is similar, but much more involved. Hence, it is instructive to first go through the analysis of asynchronous $\svrg$. 
\begin{theorem}
Suppose step size $\eta$, epoch size $m$ are chosen such that the following condition holds:
$$
0 < \theta_s := \frac{ \left(\frac{1}{\lambda\eta m} + 4L\left( \frac{\eta+ L\Delta\tau^2\eta^2}{1 - 2L^2 \Delta\eta^2\tau^2}\right) \right)}{\left(1 - 4L\left( \frac{\eta + L\Delta\tau^2\eta^2}{1 - 2L^2 \Delta\eta^2\tau^2}\right)\right)} < 1.
$$
Then, for the iterates of an asynchronous variant of Algorithm~\ref{alg:gsvrg} with $\svrg$ schedule and probabilities $p_i = 1/m$ for all $i \in [m]$, we have
\begin{align*}
\mathbb{E}[f(\tilde{x}^{k+1}) - f(x^*)] \leq \theta_s \ \mathbb{E}[f(\tilde{x}^{k}) - f(x^*)].
\end{align*}
\label{thm:t2} 
\end{theorem}
The bound obtained in Theorem~\ref{thm:t2} is useful when $\Delta$ is small. To see this, as earlier, consider the indicative case where $L/\lambda = n$. The synchronous version of $\svrg$ obtains a convergence rate of $\theta = 0.5$ for step size $\eta = 0.1/L$ and epoch size $m = O(n)$. For the asynchronous variant of $\svrg$, by setting $\eta = 0.1/2(\max\{1,\Delta^{1/2} \tau\}L)$, we obtain a similar rate with $m = O(n + \Delta^{1/2}\tau n)$. To obtain this, set $\eta = \rho/L$ where $\rho = 0.1/2(\max\{1,\Delta^{1/2} \tau\})$ and $\theta_s = 0.5$. Then, a simple calculation gives the following:
\begin{align*}
\frac{m}{n} = \frac{2}{\rho} \left(\frac{1 - 2\Delta\tau^2\rho^2}{1 - 12\rho - 14\Delta\tau^2\rho^2}\right) \leq c'\max\{1,\Delta^{1/2} \tau\},
\end{align*}
where $c'$ is some constant. This follows from the fact that $\rho=0.1/2(\max\{1,\Delta^{1/2}\tau\})$. Suppose $\tau < 1/\Delta^{1/2}$. Then we can achieve nearly the same guarantees as the synchronous version, but $\tau$ times faster since we are running the algorithm asynchronously. For example, consider the sparse setting where $\Delta = o(1/n)$; then it is possible to get near linear speedup when $\tau = o(n^{1/2})$. On the other hand, when $\Delta^{1/2}\tau > 1$, we can obtain a theoretical speedup of $1/\Delta^{1/2}$.

We finally provide the convergence result for the asynchronous algorithm in the general case. The proof is complicated by the fact that set $A$, unlike in $\svrg$, changes during the epoch. The key idea is that only a single element of $A$ changes at each iteration. Furthermore, it can only change to one of the iterates in the epoch. This control provides a handle on the error obtained due to the staleness. Due to space constraints, the proof is relegated to the appendix.

\begin{theorem}
For any positive parameters $c, \beta, \kappa > 1$, step size $\eta$ and epoch size $m$, we define the following quantities:
\begin{align*}
\zeta &= \left(c\eta^2 + \left(1 - \frac{1}{\kappa} \right)^{-\tau}cL\Delta\tau^2\eta^3 \right), \\
\gamma_a &= \kappa \left[1 - \left(1 - \frac{1}{\kappa}\right)^m \right] \left[2c\eta - 8\zeta L(1+\beta) - \frac{2c}{\kappa\lambda} - \frac{96\zeta L\tau}{n}  \left(1 - \frac{1}{\kappa}\right)^{-\tau} - \frac{1}{n} \right], \\
\theta_a &= \max \left\lbrace \left[\frac{2c}{\gamma_a\lambda}\left(1 - \frac{1}{\kappa}\right)^m + \frac{8\zeta L\left(1 + \frac{1}{\beta}\right)}{\gamma_a}  \kappa \left[1 - \left(1 - \frac{1}{\kappa}\right)^m \right]  \right], \left(1 - \frac{1}{\kappa}\right)^m  \right\rbrace.
\end{align*}
Suppose probabilities $p_i \propto (1 - \frac{1}{\kappa})^{m-i}$, parameters  $\beta, \kappa$, step-size $\eta$, and epoch size $m$ are chosen such that the following conditions are satisfied:
$$
\frac{1}{\kappa} + 8\zeta L\left(1 + \frac{1}{\beta}\right) + \frac{96\zeta L\tau}{n}  \left(1 - \frac{1}{\kappa}\right)^{-\tau} \leq \frac{1}{n}, \ \eta^2 \leq \left(1 - \frac{1}{\kappa}\right)^{m-1}\frac{1}{12L^2\Delta\tau^2}, \ \gamma_a > 0 , \ \theta_a < 1.
$$
Then, for the iterates of asynchronous variant of Algorithm~\ref{alg:gsvrg} with $\svag$ schedule we have
\begin{align*}
& \mathbb{E}\left[f(\tilde{x}^{k+1}) - f(x^*) + \frac{1}{\gamma_a} \tilde{G}_{k+1}\right] \leq  \theta_a  \mathbb{E}\left[f(\tilde{x}^{k}) - f(x^*) + \frac{1}{\gamma_a} \tilde{G}_{k}\right].
\end{align*}
\label{thm:t3}
\end{theorem}

\begin{corollary}
Note that $\tilde{G}_k \geq 0$ and therefore, under the conditions specified in Theorem~\ref{thm:t3} and with $\bar{\theta}_a = \theta_a \left(1 + 1/\gamma_a\right) < 1$, we have
\begin{align*}
\mathbb{E}\left[f(\tilde{x}^{k}) - f(x^*)\right] \leq  \bar{\theta}_a^k \ \left[f(x^{0}) - f(x^*) \right].
\end{align*}
\end{corollary}
By using step size normalized by $\Delta^{1/2} \tau$ (similar to Theorem~\ref{thm:t2}) and parameters similar to the ones specified after Theorem~\ref{thm:t1} we can show speedups similar to the ones obtained in Theorem~\ref{thm:t2}. Please refer to the appendix for more details on the parameters in Theorem~\ref{thm:t3}.

Before ending our discussion on the theoretical analysis, we would like to highlight an important point. Our emphasis throughout the paper was on generality. While the results are presented here in full generality, one can obtain stronger results in specific cases. For example, in the case of $\saga$, one can obtain \emph{per iteration} convergence guarantees (see \cite{Defazio14}) rather than those corresponding to \emph{per epoch} presented in the paper. Also, $\saga$ can be analyzed without any additional synchronization per epoch. However, there is no qualitative difference in these guarantees accumulated over the epoch. Furthermore, in this case, our analysis for both synchronous and asynchronous cases can be easily modified to obtain convergence properties similar to those in \cite{Defazio14}.

\section{Experiments}
\label{sec:expt}
We present our empirical results in this section. For our experiments, we study the problem of binary classification via  $l_2$-regularized logistic regression. More formally, we are interested in the following optimization problem:
\begin{align}
\min_{x} \frac{1}{n}\sum_{i=1}^n \left(\log(1+\exp(y_i z_i^{\top} x)) + \lambda\|x\|^2\right),
\label{eq:logistic}
\end{align}
where $z_i \in \mathbb{R}^d$ and $y_i$ is the corresponding label for each $i \in [n]$. In all our experiments, we set $\lambda = 1/n$. Note that such a choice leads to high condition number. 

A careful implementation of \svrg \ is required for sparse gradients since the implementation as stated in Algorithm~\ref{alg:gsvrg} will lead to dense updates at each iteration. For an efficient implementation, a scheme like the `just-in-time' update scheme, as suggested in \cite{Schmidt13}, is required. Due to lack of space, we provide the implementation details in the appendix.

\begin{figure*}[!t]
\centering
   \begin{minipage}[b]{.24\textwidth}
   \includegraphics[width=\textwidth]{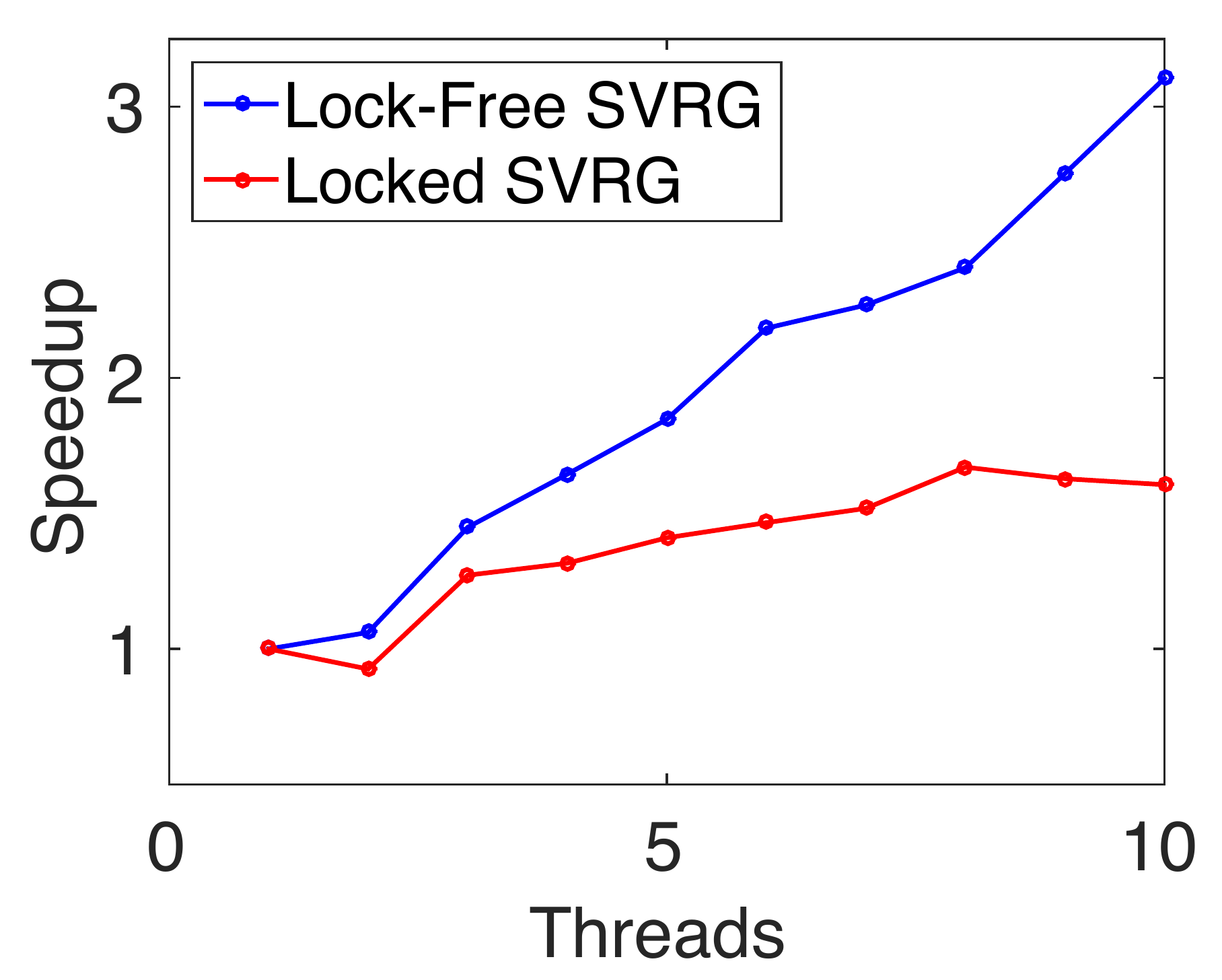}
   \end{minipage} %
   \begin{minipage}[b]{.24\textwidth}
   \includegraphics[width=\textwidth]{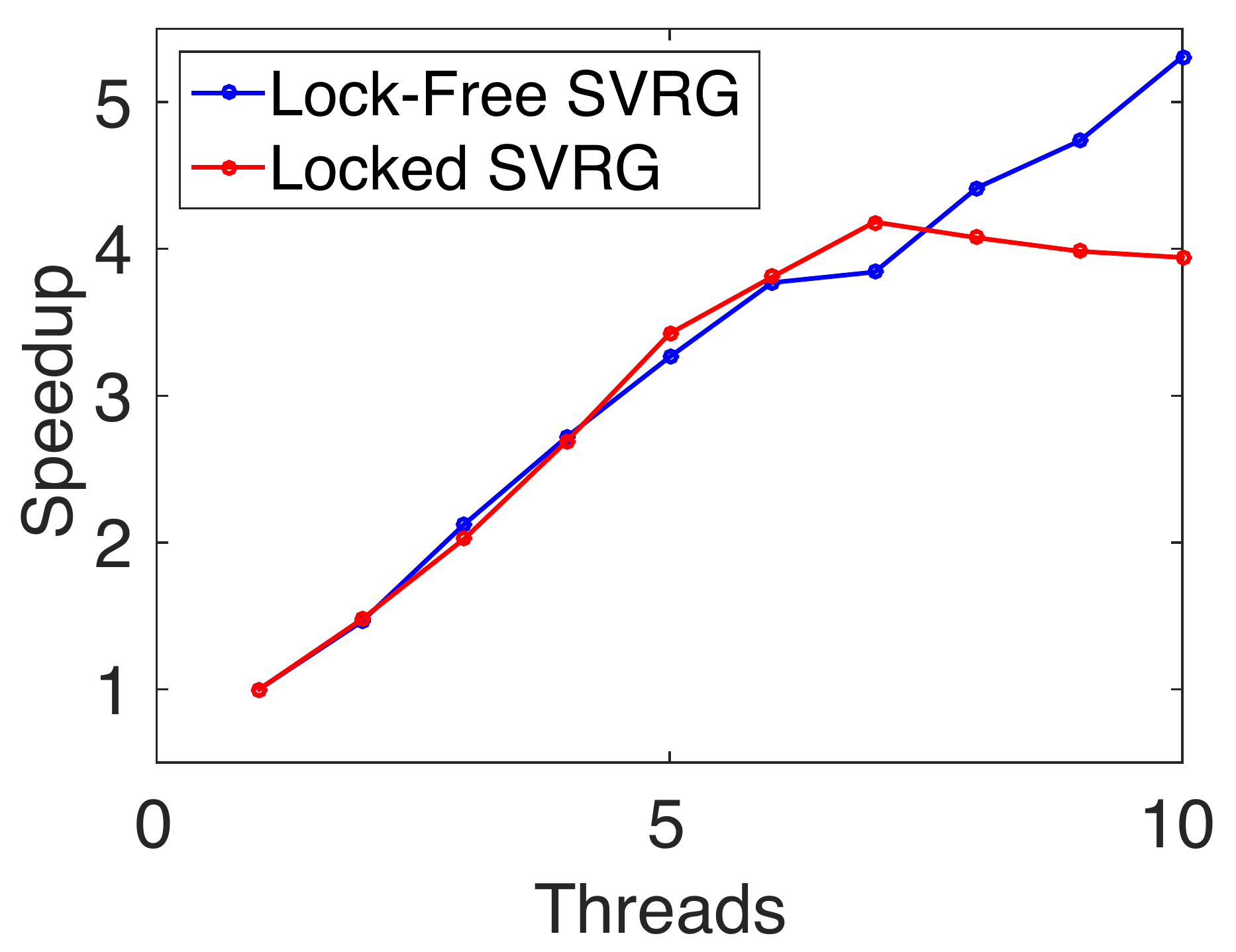}
   \end{minipage}
   \begin{minipage}[b]{.24\textwidth}
   \includegraphics[width=\textwidth]{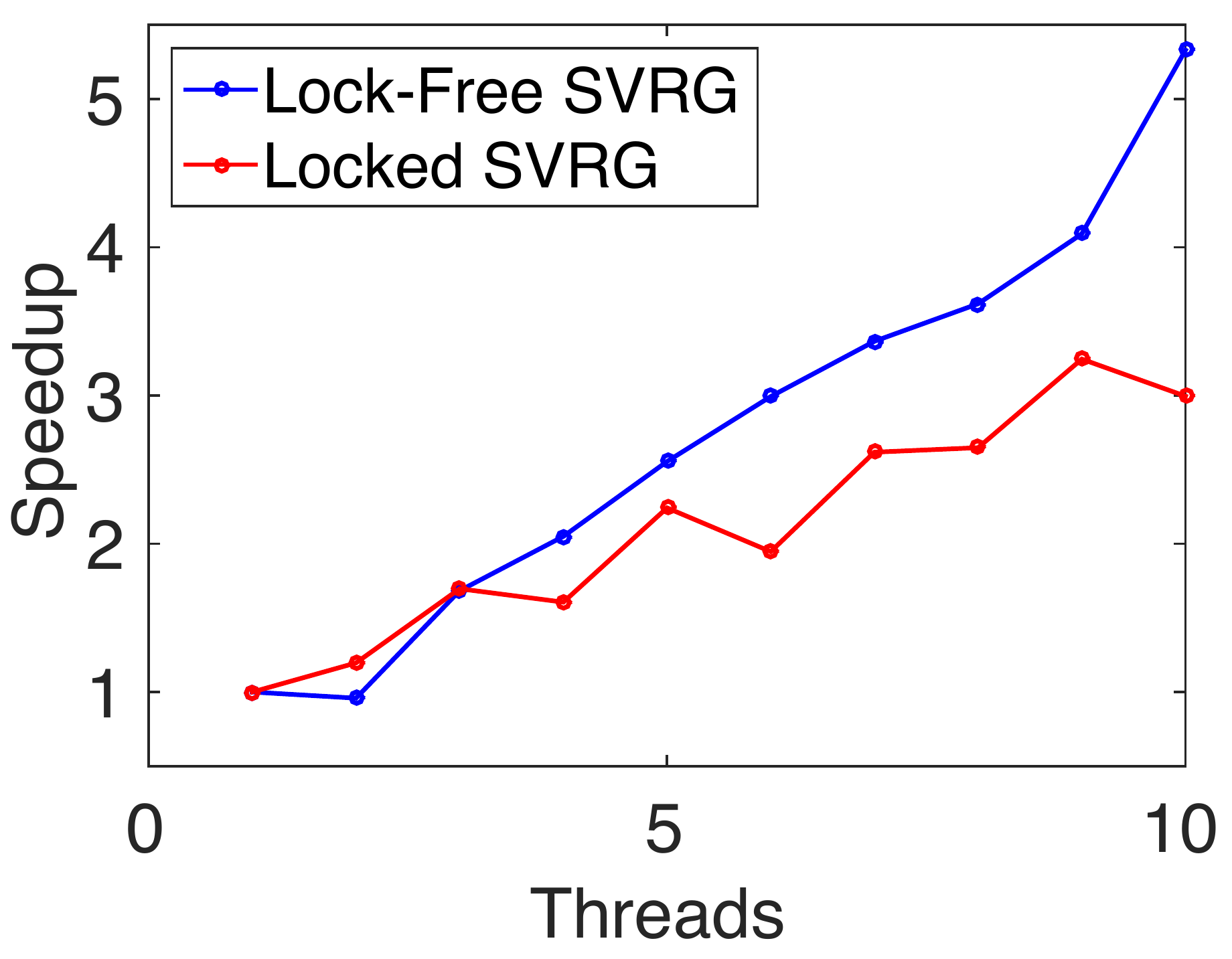}
   \end{minipage}
   \begin{minipage}[b]{.24\textwidth}
   \includegraphics[width=\textwidth]{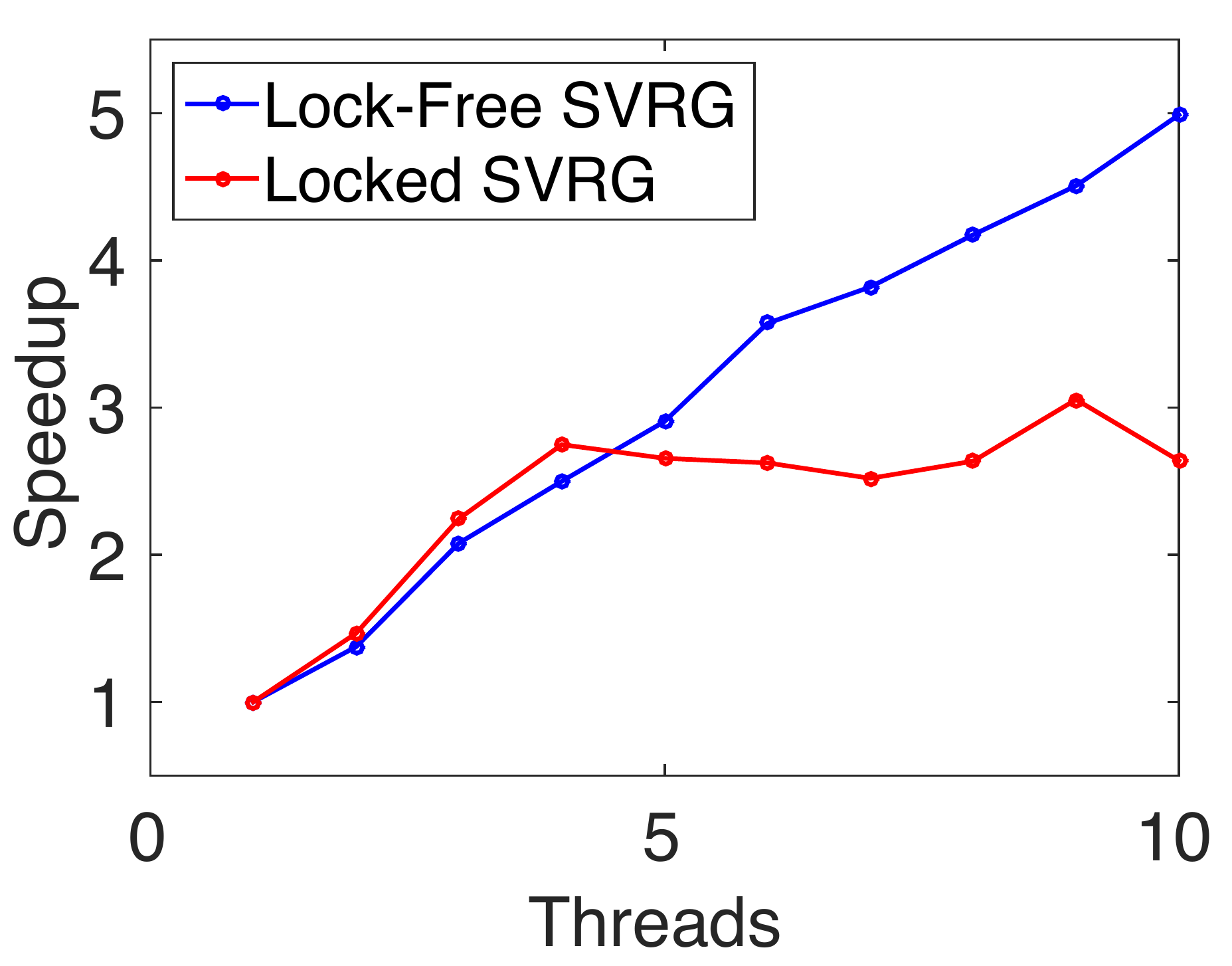}
   \end{minipage}
	\caption{$l_2$-regularized logistic regression. Speedup curves for Lock-Free $\svrg$ and Locked $\svrg$ on rcv1 (left), real-sim (left center), news20 (right center) and url (right) datasets. We report the speedup achieved by increasing the number of threads.}
	\label{fig:results1}
\end{figure*}

We evaluate the following algorithms for our experiments:
\begin{itemize}[leftmargin=+.2in]
\item {\bf Lock-Free $\svrg$}: This is the lock-free asynchronous variant of Algorithm~\ref{alg:gsvrg} using $\svrg$ schedule; all threads can read and update the parameters with any synchronization. Parameter updates are performed through atomic \emph{compare-and-swap} instruction \cite{Recht11}. A constant step size that gives the best convergence is chosen for the dataset.
\item {\bf Locked $\svrg$}: This is the locked version of the asynchronous variant of Algorithm~\ref{alg:gsvrg} using $\svrg$ schedule. In particular, we use a \emph{concurrent read exclusive write} locking model, where all threads can read the
  parameters but only one threads can update the parameters at a given time. The step size is chosen similar to Lock-Free $\svrg$.
\item {\bf Lock-Free $\sgd$}: This is the lock-free asynchronous variant of the $\sgd$ algorithm (see \cite{Recht11}). We compare two different versions of this algorithm: (i) $\sgd$ with constant step size (referred to as $\csgd$). (ii) $\sgd$ with decaying step size $\eta_0 \sqrt{\sigma_0/(t + \sigma_0)}$ (referred to as $\dsgd$), where constants $\eta_0$ and $\sigma_0$ specify the scale and speed of decay. For each of these versions, step size is tuned for each dataset to give the best convergence progress.
\end{itemize}
All the algorithms were implemented in C++ \footnote{All experiments were conducted on a Google Compute Engine n1-highcpu-32 machine with 32 processors and 28.8 GB RAM.}. We run our experiments on datasets from LIBSVM website\footnote{\url{http://www.csie.ntu.edu.tw/~cjlin/libsvmtools/datasets/binary.html}}. Similar to \cite{Xiao14}, we normalize each example in the dataset so that $\|z_i\|_2 = 1$ for all $i \in [n]$. Such a normalization leads to an upper bound of $0.25$ on the Lipschitz constant of the gradient of $f_i$. The epoch size $m$ is chosen as $2n$ (as recommended in \cite{Johnson13}) in all our experiments.
In the first experiment, we compare the speedup achieved by our asynchronous algorithm. To this end, for each dataset we first measure the time required for the algorithm to each an accuracy of $10^{-10}$ (i.e., $f(x) - f(x^*) < 10^{-10}$). The speedup with $P$ threads is defined as the ratio of the runtime with a single thread to the runtime with $P$ threads. Results in Figure~\ref{fig:results1} show the speedup on various datasets. As seen in the figure, we achieve significant speedups for all the datasets. Not surprisingly, the speedup achieved by Lock-free $\svrg$ is much higher than ones obtained by locking. Furthermore, the lowest speedup is achieved for rcv1 dataset. Similar speedup behavior was reported for this dataset in \cite{Recht11}. It should be noted that this dataset is not sparse and hence, is a bad case for the algorithm (similar to \cite{Recht11}).

For the second set of experiments we compare the performance of Lock-Free $\svrg$ with stochastic gradient descent. In particular, we compare with the variants of stochastic gradient descent, $\dsgd$ and $\csgd$, described earlier in this section. It is well established that the performance of variance reduced stochastic methods is better than that of $\sgd$. We would like to empirically verify that such benefits carry over to the asynchronous variants of these algorithms. Figure~\ref{fig:results2} shows the performance of Lock-Free $\svrg$, $\dsgd$ and $\csgd$. Since the computation complexity of each epoch of these algorithms is different, we directly plot the objective value versus the runtime for each of these algorithms. We use 10 cores for comparing the algorithms in this experiment. As seen in the figure, Lock-Free $\svrg$ outperforms both $\dsgd$ and $\csgd$. The performance gains are qualitatively similar to those reported in \cite{Johnson13} for the synchronous versions of these algorithms. It can also be seen that the $\dsgd$, not surprisingly, outperforms $\csgd$ in all the cases. In our experiments, we observed that Lock-Free $\svrg$, in comparison to $\sgd$, is relatively much less sensitive to the step size and more robust to increasing threads.

\begin{figure*}[!t]
\centering
   \begin{minipage}[b]{.24\textwidth}
   \includegraphics[width=\textwidth]{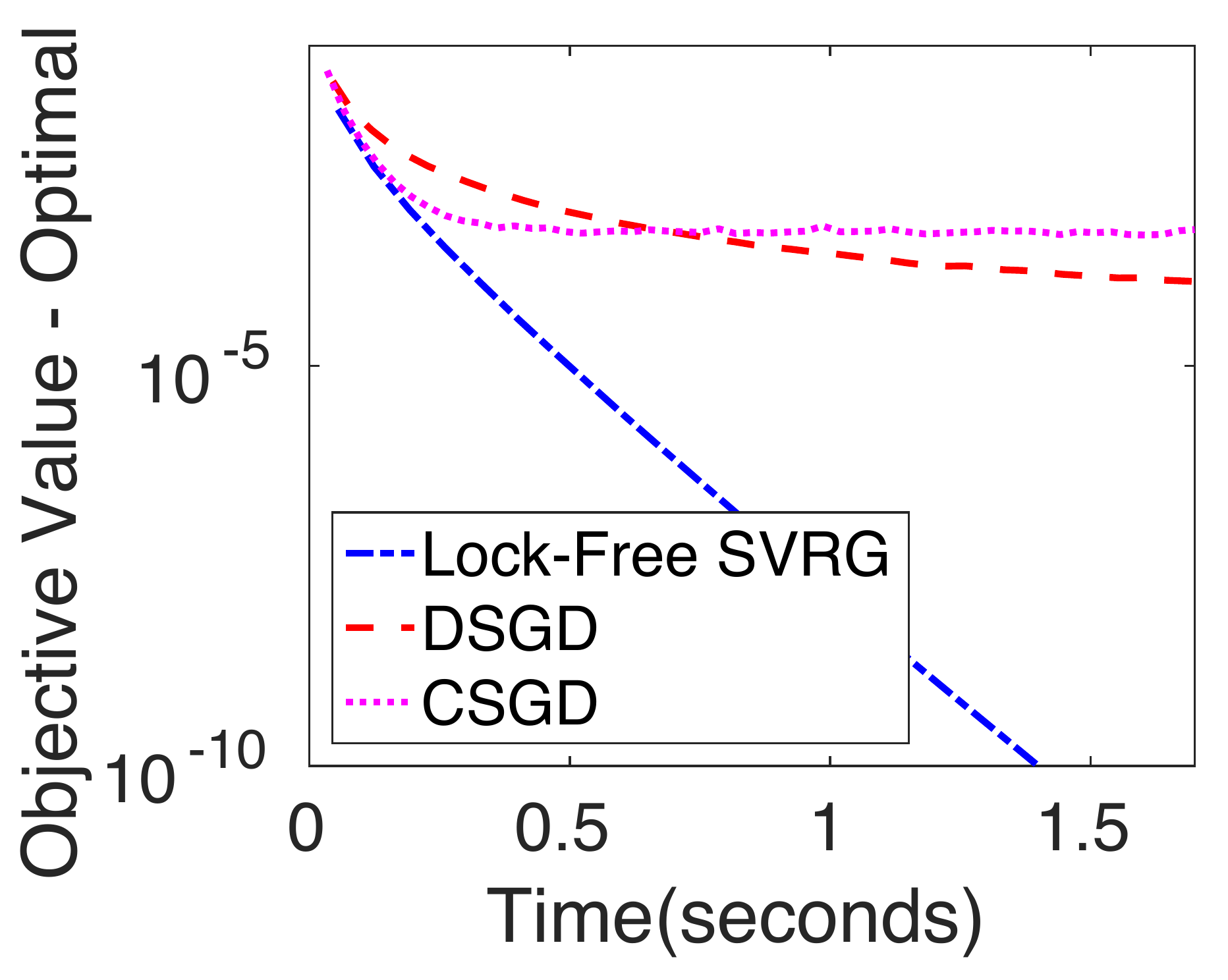}
   \end{minipage} %
   \begin{minipage}[b]{.24\textwidth}
   \includegraphics[width=\textwidth]{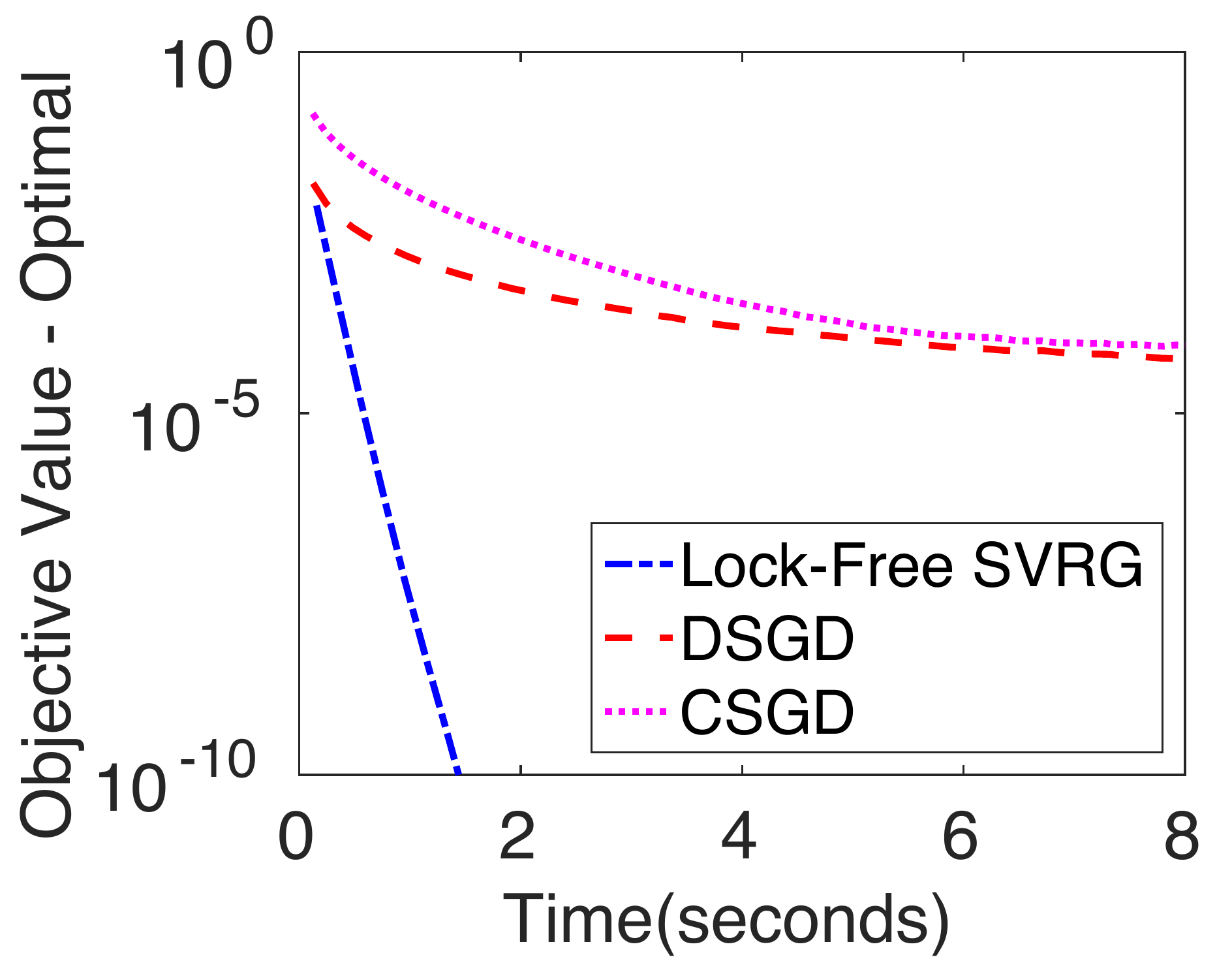}
   \end{minipage}
   \begin{minipage}[b]{.24\textwidth}
   \includegraphics[width=\textwidth]{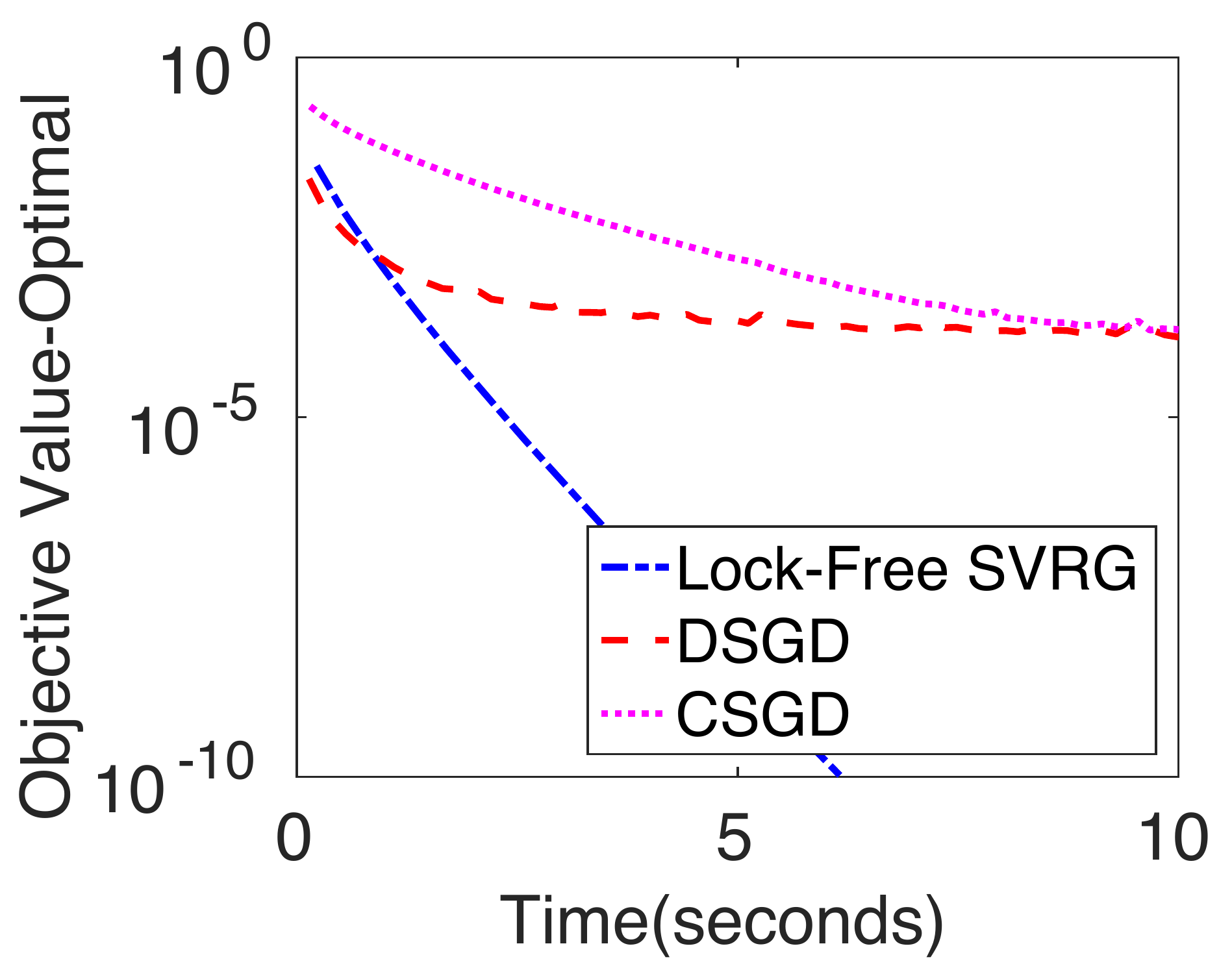}
   \end{minipage}
   \begin{minipage}[b]{.24\textwidth}
   \includegraphics[width=\textwidth]{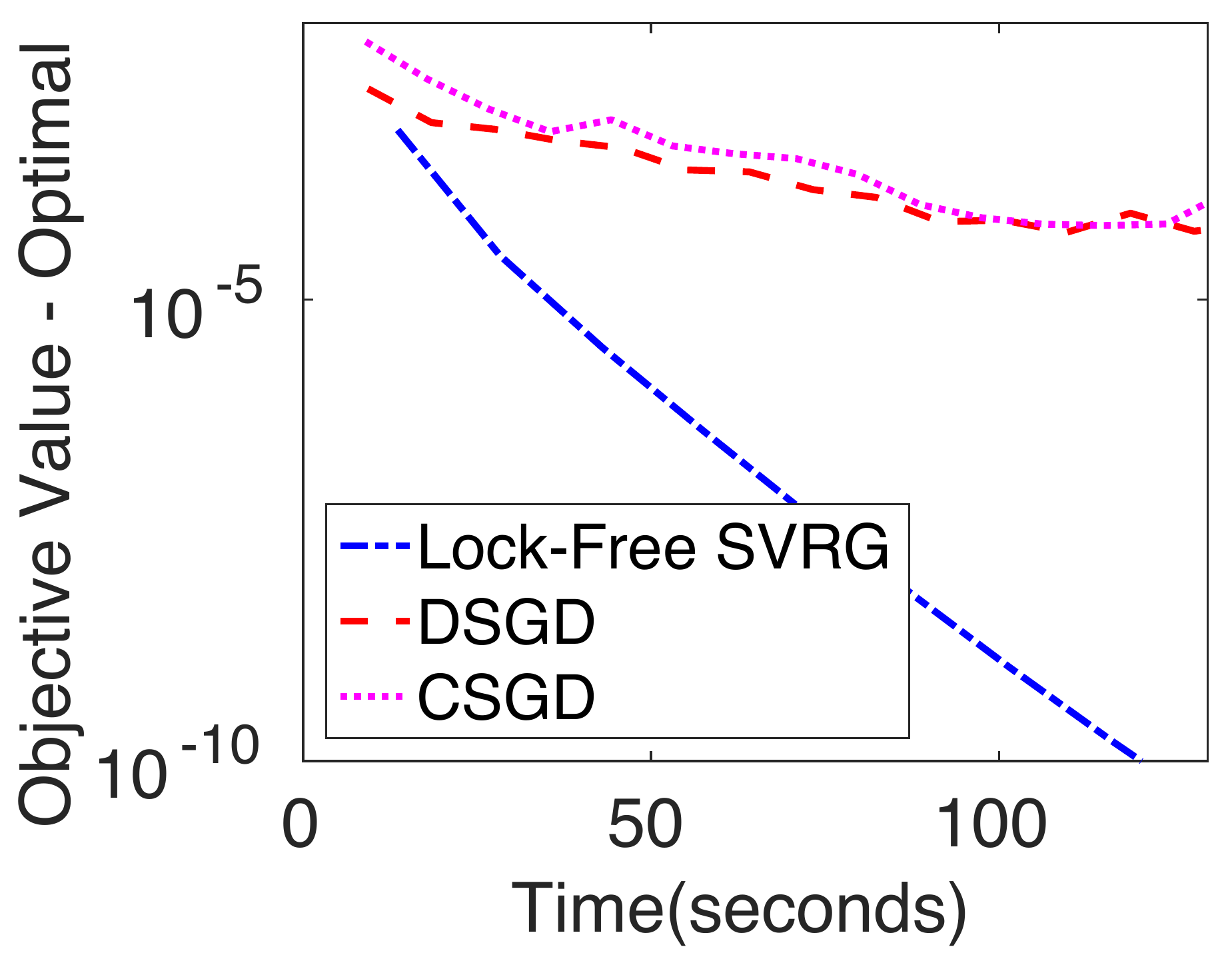}
   \end{minipage}
   \caption{$l_2$-regularized logistic regression. Training loss residual $f(x) - f(x^*)$ versus time plot of Lock-Free $\svrg$, $\dsgd$ and $\csgd$ on rcv1 (left), real-sim (left center), news20 (right center) and url (right) datasets. The experiments are parallelized over 10 cores.}
   \label{fig:results2}
\end{figure*}

\section{Discussion \& Future Work}

In this paper, we presented a unifying framework based on \cite{Defazio14}, that captures many popular variance reduction techniques for stochastic gradient descent. We use this framework to develop a simple hybrid variance reduction method. The primary purpose of the framework, however, was to provide a common platform to analyze various variance reduction techniques. To this end, we provided convergence analysis for the framework under certain conditions. More importantly, we propose an asynchronous algorithm for the framework with provable convergence guarantees. The key consequence of our approach is that we obtain asynchronous variants of several algorithms like $\svrg$, $\saga$ and $\stgd$. Our asynchronous algorithms exploits sparsity in the data to obtain near linear speedup in settings that are typically encountered in machine learning.

For future work, it would be interesting to perform an empirical comparison of various schedules. In particular, it would be worth exploring the space-time-accuracy tradeoffs of these schedules. We would also like to analyze the effect of these tradeoffs on the asynchronous variants.

{\small
\bibliographystyle{abbrv}
\setlength{\bibsep}{0pt}
\bibliography{bibfile}
}
\clearpage

\appendix

\section{Appendix}

{\bf Notation}: We use $D_f$ to denote the Bregman divergence (defined below) for function $f$.
$$
D_f(x,y) = f(x) - f(y) - \langle \nabla f(y), x - y \rangle.
$$ 
For ease of exposition, we use $\mathbb{E}[X]$ to denote the expectation the random variable $X$ with respect to indices $\{i_1, \dots, i_t\}$ when $X$ depends on just these indices up to step $t$. This dependence will be clear from the context. We use $\mathbbm{1}$ to denote the indicator function. We assume $\sum_{d = i}^j a_d = 0$ if $i > j$.

We would like to clarify the definition of $x^{km}$ here. As noted in the main text, we assume that $x^{km+m}$ is replaced with an element chosen randomly from $\{x^{km},\dots,x^{km+m-1}\}$ with probability $\{p_1, \cdots, p_m\}$ at the end of the $(k+1)^{\text{th}}$ epoch. However, whenever $x^{km}$ appears in the analysis (proofs), it represents the iterate before this replacement. 

\section*{Implementation Details}

Since we are interested in sparse datasets, simply taking $f_i(x) = \log(1+\exp(y_i z_i^{\top} x)) + \lambda\|x\|^2$ is not efficient as it requires updating the whole vector $x$ at each iteration. This is due to the  regularization term in each of the $f_i$'s. Instead, similar to \cite{Recht11}, we rewrite problem in~\eqref{eq:logistic} as follows:
\begin{align*}
\min_{x} \frac{1}{n}\sum_{i=1}^n \left(\log(1+\exp(y_i z_i^{\top} x)) + \lambda \sum_{j \in nz(z_i)} \frac{\|x_j\|^2}{d_j}\right),
\end{align*}
where $nz(z)$ represents the non-zero components of vector $z$, and $d_j = \sum_i \mathbbm{1}(j \in nz(z_i))$. While this leads to sparse gradients at each iteration, updates in $\svrg$ are still dense due to the part of the update that contains $\sum_i \nabla f_i(\alpha_i)/n$. This problem can be circumvented by using the following update scheme. First, recall that for $\svrg$, $\sum_i \nabla f_i(\alpha_i)/n$ does not change during an epoch (see Figure~\ref{fig:all-sched}). Therefore, during the $(k+1)^{\text{st}}$ epoch we have the following relationship:
\begin{align*}
x^t = \left[\tilde{x}^k - \eta \sum_{j=km}^{t-1} (\nabla f_{i_{j}}(x^{j}) - \nabla f_{i_{j}}(\tilde{x}^{k}))\right] - \left[\frac{\eta(t - km)}{n} \sum_{i=1}^n \nabla f_i(\tilde{x}^k)\right].
\end{align*}
We maintain each bracketed term separately. The updates to the first term in the above equation are sparse while those to the second term are just a simple scalar addition, since we already maintain the average gradient $\sum_{i=1}^n \nabla f_i(\tilde{x}^k)/n$. When the gradient of $f_{i_t}$ at $x^t$ is needed, we \emph{only} calculate components of $x^t$ required for $f_{i_t}$ on the fly by aggregating these two terms. Hence, each step of this update procedure can be implemented in a way that respects sparsity of the data.

\section*{Proof of Theorem~\ref{thm:t1}}
\begin{proof}
We expand function $f$ as $f(x) = g(x) + h(x)$ where $g(x) = \frac{1}{n} \sum_{i \in S} f_i(x)$ and $h(x) = \frac{1}{n} \sum_{i \notin S} f_i(x)$. Let the present epoch be $k+1$. We define the following:
\begin{align*}
v^t &= \frac{1}{\eta} (x^{t+1} - x^t) = -\left[ \nabla f_{i_t}(x^t) - \nabla f_{i_t}(\alpha_{i_t}^t) + \frac{1}{n} \sum_i \nabla f_i(\alpha_i^t) \right] \\
G_t &= \frac{1}{n} \sum_{i \in S} \left(f_{i}(\alpha_i^{t}) - f_i(x^*) - \langle \nabla f_i(x^*), \alpha_i^{t} - x^* \rangle \right) \\
R_t &= \mathbb{E}\left[ c\|x^{t} - x^*\|^2 + G_t \right].
\end{align*}
We first observe that $\mathbb{E}[v^t] = -\nabla f(x^t)$. This follows from the unbiasedness of the gradient at each iteration. Using this observation, we have the following:
\begin{align}
\mathbb{E}[R_{t+1}] & = \mathbb{E}[c\|x^{t+1} - x^*\|^2 + G_{t+1}] = \mathbb{E}[c\|x^t + \eta v^t - x^*\|^2 + G_{t+1}] \nonumber \\
&= c \mathbb{E}\left[ \|x^t - x^*\|^2\right] + c\eta^2 \mathbb{E}\left[\|v^t\|^2\right] + 2c\eta \mathbb{E}\left[\langle x^t - x^*, v^t \rangle \right] + \mathbb{E}[G_{t+1}]  \nonumber \\
&\leq  c \mathbb{E}\left[ \|x^t - x^*\|^2\right] + c\eta^2 \mathbb{E}\left[\|v^t\|^2\right] - 2c\eta \mathbb{E}\left[f(x^t) - f(x^*)\right] + \mathbb{E}[G_{t+1}].
\label{eq:thm1-eq1}
\end{align}
The last step follows from convexity of $f$ and the unbiasedness of $v^t$. We have the following relationship between $G_{t+1}$ and $G_t$.
\begin{align}
\mathbb{E}[G_{t+1}] &= \left(1 - \frac{1}{n}\right) \mathbb{E} \left[G_t \right] + \frac{1}{n} \mathbb{E} \left[\frac{1}{n} \sum_{i \in S} \left(f_{i}(x^{t}) - f_i(x^*) - \langle \nabla f_i(x^*), x^{t} - x^* \rangle \right) \right] \nonumber \\
&= \left(1 - \frac{1}{n}\right) \mathbb{E} \left[G_t \right] + \frac{1}{n}\mathbb{E}[D_g(x^t, x^*)]. \label{eq:thm1-gt-rec}
\end{align}
This follows from the definition of the schedule of $\svag$ for indices  in $S$. Substituting the above relationship in Equation~\eqref{eq:thm1-eq1} we get the following.
\begin{align*}
R_{t+1} & \leq R_{t} + c\eta^2 \mathbb{E}\left[\|v^t\|^2\right] - 2c\eta \mathbb{E}\left[f(x^t) - f(x^*)\right] - \frac{1}{n} \mathbb{E}[G_t] + \frac{1}{n}\mathbb{E}[D_g(x^t, x^*)] \\
&\leq \left(1 - \frac{1}{\kappa}\right) R_t + \frac{c}{\kappa} \mathbb{E}[\|x^t - x^*\|^2] + c\eta^2 \mathbb{E}\left[\|v^t\|^2\right] - 2c\eta \mathbb{E}\left[f(x^t) - f(x^*)\right] \\
& \qquad \qquad + \left(\frac{1}{\kappa} - \frac{1}{n}\right) \mathbb{E}[G_t] + \frac{1}{n}\mathbb{E}[D_g(x^t, x^*)] \\
&:=  \left(1 - \frac{1}{\kappa}\right) R_t + b_t.
\end{align*}
We describe the bounds for $b_t$ (defined below).
\begin{align*}
b_t &= \frac{c}{\kappa} \underbrace{\mathbb{E}[\|x^t - x^*\|^2]}_{T_1} + c\eta^2\underbrace{ \mathbb{E}\left[\|v^t\|^2\right]}_{T_2} - 2c\eta \mathbb{E}\left[f(x^t) - f(x^*)\right] \\
& \qquad \qquad + \left(\frac{1}{\kappa} - \frac{1}{n}\right) \mathbb{E}[G_t] + \frac{1}{n}\mathbb{E}[D_g(x^t, x^*)].
\end{align*}
The terms $T_1$ and $T2$ can be bounded in the following fashion:
\begin{align*}
T_1 &= \mathbb{E}[\|x^t - x^*\|^2] \leq \frac{2}{\lambda}\mathbb{E}[f(x^t) - f(x^*)] \\
T_2 &= \mathbb{E}\left[\|v^t\|^2\right] \leq \left(1 + \frac{1}{\beta}\right) \mathbb{E}\left[\|\nabla f_{i_t}(\alpha_{i_t}^t) - \nabla f_{i_t}(x^*)\|^2 \right] + (1 + \beta) \mathbb{E}\left[\|\nabla f_{i_t}(x^t) - \nabla f_{i_t}(x^*)\|^2 \right] \\
& \qquad \qquad \ \ \ \ \ \ \leq \frac{2L}{n}\left(1 + \frac{1}{\beta}\right) \mathbb{E}\sum_i \left[f_i(\alpha_{i}^t) - f(x^*) - \left\langle \nabla f_i(x^*), \alpha_{i}^t - x^* \right\rangle \right] \\
& \qquad \qquad \qquad \ \ \ \ \ \ + \frac{2L}{n}(1 + \beta) \mathbb{E}\sum_i \left[f_i(x^t) - f(x^*)\right] \\
&\qquad \qquad \ \ \ \ \ \ \leq 2L\left(1 + \frac{1}{\beta}\right)\mathbb{E}\left[G_{t} + D_h(\tilde{x}^k,x^*)\right] + 2L(1+\beta) \mathbb{E}[f(x^t) - f(x^*)].
\end{align*}
The bound on $T_1$ is due to strong convexity nature of function $f$. The first inequality and second inequalities on $T_2$ directly follows from Lemma 3 of \cite{Defazio14} and simple application of Lemma~\ref{lem:var-lemma} respectively. The third inequality follows from the definition of $G_t$ and the fact that $\alpha_i^t = \tilde{x}^k$ for all $i \notin S$ and $t \in \{km, \dots, km+m-1\}$.

Substituting these bounds $T_1$ and $T_2$ in $b_t$, we get
\begin{align}
b_t &\leq -\left[2c\eta - 2cL\eta^2(1+\beta) - \frac{2c}{\kappa\lambda}\right] \mathbb{E}\left[f(x^t) - f(x^*) \right] \nonumber \\
& \qquad  + \left(\frac{1}{\kappa} + 2cL\eta^2\left(1 + \frac{1}{\beta}\right) - \frac{1}{n}\right) \mathbb{E}[G_t] + \frac{1}{n}\mathbb{E}[D_g(x^t, x^*)] \nonumber \\
& \qquad + 2cL\eta^2\left(1 + \frac{1}{\beta}\right)\mathbb{E}\left[ D_h(\tilde{x}^k,x^*)\right] \nonumber \\
&\leq -\left[2c\eta - 2cL\eta^2(1+\beta) - \frac{1}{n} -  \frac{2c}{\kappa\lambda}\right] \mathbb{E}\left[f(x^t) - f(x^*) \right] \nonumber \\
& \qquad  + \left(\frac{1}{\kappa} + 2cL\eta^2\left(1 + \frac{1}{\beta}\right) - \frac{1}{n}\right) \mathbb{E}[G_t] + 2cL\eta^2\left(1 + \frac{1}{\beta}\right)\mathbb{E}\left[ D_h(\tilde{x}^k,x^*)\right] \nonumber \\
&\leq -\left[2c\eta - 2cL\eta^2(1+\beta) - \frac{1}{n} -  \frac{2c}{\kappa\lambda}\right] \mathbb{E}\left[f(x^t) - f(x^*) \right] + 2cL\eta^2\left(1 + \frac{1}{\beta}\right)\mathbb{E}\left[ D_h(\tilde{x}^k,x^*)\right]. \label{eq:thm1-bt-bound}
\end{align}
The second inequality follows from Lemma~\ref{lem:breg}. In particular, we use the fact that $f(x) - f(x^*) = D_f(x,x^*)$ and $D_f(x,x^*) = D_g(x,x^*) + D_h(x,x^*) \geq D_g(x,x^*)$. The third inequality follows from the following for the choice of our parameters:
$$
\frac{1}{\kappa} + 2Lc\eta^2 \left(1 + \frac{1}{\beta}\right) \leq \frac{1}{n}.
$$
Applying the recursive relationship on $R_{t+1}$ for m iterations, we get
\begin{align*}
R_{km+m} \leq \left(1 - \frac{1}{\kappa}\right)^m \tilde{R}_{k} + \sum_{j=0}^{m-1}  \left(1 - \frac{1}{\kappa}\right)^{m-1 - j} b_{km+j}
\end{align*}
where 
$$
\tilde{R}_{k} = \mathbb{E}\left[ c\|\tilde{x}^{k} - x^*\|^2 + \tilde{G}_k \right].
$$
Substituting the bound on $b_t$ from Equation~\eqref{eq:thm1-bt-bound} in the above equation we get the following inequality:
\begin{align*}
R_{km+m} &\leq \left(1 - \frac{1}{\kappa}\right)^m \tilde{R}_{k} \\
& \ - \sum_{j=0}^{m-1} \left(2c\eta (1 - L\eta(1+\beta)) - \frac{1}{n} - \frac{2c}{\kappa\lambda} \right) \left(1 - \frac{1}{\kappa}\right)^{m-1 - j} \mathbb{E}\left[f(x^{km+j}) - f(x^*)\right] \\
& \ +  \sum_{j=0}^{m-1} \left(1 - \frac{1}{\kappa}\right)^{m-1 - j} 2Lc\eta^2 \left(1 + \frac{1}{\beta}\right)\mathbb{E}\left[h(\tilde{x}^{k}) - h(x^*) - \langle \nabla h(x^*), \tilde{x}^k - x^*\rangle \right].
\end{align*}
We now use the fact that $\tilde{x}^{k+1}$ is chosen randomly from $\{x^{km}, \dots, x^{km+m-1}\}$ with probabilities proportional to $\{(1-1/\kappa)^{m-1}, \dots, 1\}$ we have the following consequence of the above inequality.
\begin{align*}
R_{km+m} &+ \kappa \left[1 - \left(1 - \frac{1}{\kappa}\right)^m \right] \left(2c\eta (1 - L\eta(1+\beta)) - \frac{1}{n} - \frac{2c}{\kappa\lambda} \right) \mathbb{E}\left[f(\tilde{x}^{k+1}) - f(x^*)\right] \\
& \leq \frac{2c}{\lambda}\left(1 - \frac{1}{\kappa}\right)^m \mathbb{E}\left[f(\tilde{x}^k) - f(x^*)\right] + \left(1 - \frac{1}{\kappa}\right)^m \mathbb{E}\left[ \tilde{G}_k \right] \\
& \ \ + 2Lc\eta^2  \kappa \left[1 - \left(1 - \frac{1}{\kappa}\right)^m \right] \left(1 + \frac{1}{\beta}\right)\mathbb{E}\left[D_h(\tilde{x}^k,x^*)\right].
\end{align*}
For obtaining the above inequality, we used the strongly convex nature of function $f$. Again, using the Bregman divergence based inequality (see Lemma~\ref{lem:breg})
$$
f(x) - f(x^*) = D_f(x,x^*) = D_g(x,x^*) + D_h(x,x^*) \geq D_h(x,x^*),
$$
we have the following inequality
\begin{align}
& R_{km+m}+ \kappa \left[1 - \left(1 - \frac{1}{\kappa}\right)^m \right] \left(2c\eta (1 - L\eta(1+\beta)) - \frac{1}{n} - \frac{2c}{\kappa\lambda} \right) \mathbb{E}\left[f(\tilde{x}^{k+1}) - f(x^*)\right] \nonumber \\
& \leq \left[\frac{2c}{\lambda}\left(1 - \frac{1}{\kappa}\right)^m + 2Lc\eta^2  \kappa \left(1 + \frac{1}{\beta}\right)\left[1 - \left(1 - \frac{1}{\kappa}\right)^m \right]  \right] \mathbb{E}\left[f(\tilde{x}^k) - f(x^*)\right] + \left(1 - \frac{1}{\kappa}\right)^m \mathbb{E}\left[\tilde{G}_k \right]. \label{eq:thm1-rec2}
\end{align}
We use the following notation:
\begin{align*}
\gamma &= \kappa \left[1 - \left(1 - \frac{1}{\kappa}\right)^m \right] \left(2c\eta (1 - L\eta(1+\beta)) - \frac{1}{n} - \frac{2c}{\kappa\lambda} \right) \\
\theta &= \max \left\lbrace \left[\frac{2c}{\gamma\lambda}\left(1 - \frac{1}{\kappa}\right)^m + \frac{2Lc\eta^2}{\gamma}\left(1 + \frac{1}{\beta}\right)  \kappa \left[1 - \left(1 - \frac{1}{\kappa}\right)^m \right]  \right], \left(1 - \frac{1}{\kappa}\right)^m  \right\rbrace.
\end{align*}
Using the above notation, we have the following inequality from Equation~
\eqref{eq:thm1-rec2}.
\begin{align*}
& \mathbb{E}\left[f(\tilde{x}^{k+1}) - f(x^*) + \frac{1}{\gamma} \tilde{G}_{k+1}\right] \leq  \theta \  \mathbb{E}\left[f(\tilde{x}^{k}) - f(x^*) + \frac{1}{\gamma} \tilde{G}_{k}\right],
\end{align*}
where $\theta < 1$ is a constant that depends on the parameters used in the algorithm.
\end{proof}

\section*{Proof of Theorem~\ref{thm:t2}}
\begin{proof}
Let the present epoch be $k+1$. Recall that $D(t)$ denotes the iterate used in the $t^{\text{th}}$ iteration of the algorithm. We define the following:
\begin{align*}
u^t &= - \left[\nabla f_{i_t}(x^{D(t)}) - \nabla f_{i_t}(\tilde{x}^k) + \nabla f(\tilde{x}^k) \right] \\
v^t &= - \left[\nabla f_{i_t}(x^{t}) - \nabla f_{i_t}(\tilde{x}^k) + \nabla f(\tilde{x}^k) \right].
\end{align*}
We have the following:
\begin{align}
\mathbb{E}\|x^{t+1} - x^*\|^2  = \mathbb{E}\|x^t + \eta u^t - x^*\|^2 = \mathbb{E}\left[\|x^t - x^*\|^2 + \eta^2 \|u^t\|^2 + 2\eta \langle x^t - x^*, u^t \rangle \right].
\label{eq:thm2-main-bound}
\end{align}
We first bound the last term of the above inequality. We expand the term in the following manner:
\begin{align}
\mathbb{E}\langle & x^t - x^*, u^t \rangle = \mathbb{E}\left[ \langle x^* - x^t, \nabla f_{i_t}(x^{D(t)})\rangle\right] \nonumber \\
&= \underbrace{\mathbb{E}\left[\langle x^* - x^{D(t)}, \nabla f_{i_t}(x^{D(t)}) \rangle \right]}_{T_3} \nonumber \\
& \ \ + \underbrace{\sum_{d = D(t)}^{t-1}  \mathbb{E}\left[\langle x^d - x^{d+1}, \nabla f_{i_t}(x^{d}) \rangle \right]}_{T_4} + \underbrace{\sum_{d = D(t)}^{t-1}  \mathbb{E}\left[\langle x^d - x^{d+1}, \nabla f_{i_t}(x^{D(t)}) - \nabla f_{i_t}(x^{d}) \rangle \right]}_{T_5}.
\label{eq:thm2-bound}
\end{align}
The first equality directly follows from the definition of $u^t$ and its property of unbiasedness. The second step follows from simple algebraic calculations. Terms $T_3$ and $T_4$ can be bounded in the following way:
\begin{align}
T_3 &\leq \mathbb{E}[f_{i_t}(x^*) - f_{i_t}(x^{D(t)})].
\label{eq:t3-bound}
\end{align}
This bound directly follows from convexity of function $f_{i_t}$.
\begin{align}
T_4 &= \sum_{d = D(t)}^{t-1}  \mathbb{E}\left[\langle x^d - x^{d+1}, \nabla f_{i_t}(x^{d}) \rangle \right] \nonumber \\
& \leq \sum_{d = D(t)}^{t-1} \mathbb{E}\left[ f_{i_t}(x^d) - f_{i_t}(x^{d+1}) + \frac{L}{2} \|x^{d+1} - x^d\|_{i_t}^2\right] \nonumber \\
&\leq \mathbb{E}\left[ f_{i_t}(x^{D(t)}) - f_{i_t}(x^t)\right] +  \frac{L\Delta}{2} \sum_{d = D(t)}^{t-1} \mathbb{E}\left[ \|x^{d+1} - x^d\|^2\right].
\label{eq:t4-bound}
\end{align}
The first inequality follows from lipschitz continuous nature of the gradient of function $f_{i_t}$. The second inequality follows from the definition of $\Delta$. The last term $T_5$ can be bounded in the following manner.
\begin{align}
T_5 &= \mathbb{E}\left[\sum_{d=D(t)}^{t-1} \langle x^d - x^{d+1}, \nabla  f_{i_t}(x^{D(t)}) - \nabla f_{i_t}(x^d) \rangle \right] \nonumber \\
& \leq \mathbb{E}\left[\sum_{d=D(t)}^{t-1} \|x^{d+1} - x^d\|_{i_t} \|\nabla f_{i_t}(x^{D(t)}) - \nabla f_{i_t}(x^d)\| \right] \nonumber \\
&  \leq \mathbb{E}\left[\sum_{d=D(t)}^{t-1} \|x^{d+1} - x^d\|_{i_t}  \sum_{j = D(t)}^{d-1} \|\nabla f_{i_t}(x^{j+1}) - \nabla f_{i_t}(x^{j})\| \right] \nonumber \\
&  \leq \mathbb{E}\left[\sum_{d=D(t)}^{t-1} \sum_{j = D(t)}^{d-1} \frac{L}{2} \left(\|x^{d+1} - x^d\|_{i_t}^2 + \|x^{j+1} - x^j\|_{i_t}^2 \right) \right] \nonumber \\
& \leq \frac{L\Delta(\tau - 1)}{2} \mathbb{E}\sum_{d = D(t)}^{t-1} \|x^{d+1} - x^{d}\|^2.
\label{eq:t5-bound}
 \end{align}
The first inequality follows from Cauchy-Schwartz inequality. The second inequality follows from repeated application of triangle inequality. The third step is a simple application of AM-GM inequality and the fact that gradient of the function $f_{i_t}$ is lipschitz continuous. Finally, the last step can be obtained by using a simple counting argument, the fact that the staleness in gradient is at most $\tau$ and the definition of $\Delta$.

By combining the bounds on $T_3,T_4$ and $T_5$ in Equations~\eqref{eq:t3-bound}, ~\eqref{eq:t4-bound} and ~\eqref{eq:t5-bound} respectively and substituting the sum in Equation~\eqref{eq:thm2-bound}, we get
\begin{align}
\mathbb{E}\langle x^t - x^*, u^t \rangle \leq \mathbb{E}\left[f(x^*) - f(x^t) + \frac{L\Delta\tau}{2} \sum_{d = D(t)}^{t-1} \|x^{d+1} - x^{d}\|^2 \right].
\end{align}
By substituting the above inequality in Equation~\eqref{eq:thm2-main-bound}, we get
\begin{align}
\mathbb{E}\left[ \|x^{t+1} - x^*\|^2 \right] & \leq \mathbb{E}\Bigg[\|x^t - x^*\|^2 + \eta^2\|u^t\|^2 - 2\eta(f(x^t) - f(x^*)) + L\Delta\tau\eta^3 \sum_{d = D(t)}^{t-1} \|u^d\|^2\Bigg].
\label{eq:thm2-rec2}
\end{align} 
We next bound the term $\mathbb{E}[\|u^t\|^2]$ in terms of $\mathbb{E}\left[\|v^t\|^2\right]$ in the following way:
\begin{align*}
\mathbb{E}\left[\|u^t\|^2\right] &\leq 2\mathbb{E}\left[\|u^t - v^t\|^2 +\|v^t\|^2 \right] \\
&\leq 2 \mathbb{E} \left[\|\nabla f_{i_t}(x^t) - \nabla f_{i_t}(x^{D(t)})\|^2\right] + 2 \mathbb{E}\left\|v^t\|^2 \right] \\
&\leq  2 L^2 \tau \sum_{d=D(t)}^{t-1} \mathbb{E} \left[\|x^{d+1} - x^{d}\|_{i_t}^2\right] + 2 \mathbb{E}\left[\|v^t\|^2 \right] \\
&\leq  2L^2 \Delta\eta^2\tau \sum_{d=D(t)}^{t-1} \mathbb{E} \left[\|u^d\|^2\right] + 2 \mathbb{E}\left[\|v^t\|^2 \right].
\end{align*}
The first step follows from Lemma~\ref{lem:sq-lemma} for $r=2$. The third inequality follows from the lipschitz continuous nature of the gradient and simple application of Lemma~\ref{lem:sq-lemma}. Adding the above inequalities from $t=km$ to $t = km+m-1$, we get
\begin{align*}
\sum_{t=km}^{km+m-1} \mathbb{E}\left[\|u^t\|^2\right] &\leq \sum_{t=km}^{km+m-1} \left[ 2L^2 \Delta\eta^2\tau \sum_{d=D(t)}^{t-1} \mathbb{E} \left[\|u^d\|^2\right] + 2 \mathbb{E}\left[\|v^t\|^2 \right] \right] \\
& \leq 2L^2 \Delta\eta^2\tau^2 \sum_{t=km}^{km+m-1} \mathbb{E}\left[\|u^t\|^2\right] + 2 \sum_{t=km}^{km+m-1} \mathbb{E}\left[\|v^t\|^2\right].
\end{align*}
Here we again used a simple counting argument and the fact that the delay in the gradients is at most $\tau$. From the above inequality, we get
\begin{align}
\sum_{t=km}^{km+m-1} \mathbb{E}\left[\|u^t\|^2\right] \leq \frac{2}{\left(1 - 2L^2 \Delta\eta^2\tau^2\right)} \sum_{t=km}^{km+m-1} \mathbb{E}\left[\|v^t\|^2\right].
\label{eq:thm2-var-bound}
\end{align}
Adding Equation~\eqref{eq:thm2-rec2} from $t = km$ to $t = km+m-1$ and substituting Equation~\eqref{eq:thm2-var-bound} in the resultant, we get
\begin{align}
\mathbb{E}&\left[ \|x^{km+m} - x^*\|^2 \right] \leq \mathbb{E}\Bigg[\|\tilde{x}^k - x^*\|^2 + (\eta^2+ L\Delta\tau^2\eta^3)\sum_{t=km}^{km+m-1}\|u^t\|^2 - \sum_{t={km}}^{km+m-1} 2\eta (f(x^t) - f(x^*)) \Bigg] \nonumber \\
& \qquad \qquad \leq \mathbb{E}\Bigg[\|\tilde{x}^k - x^*\|^2 + 2\left( \frac{\eta^2+ L\Delta\tau^2\eta^3}{1 - 2L^2 \Delta\eta^2\tau^2}\right) \sum_{t=km}^{km+m-1} \|v^t\|^2 - \sum_{t={km}}^{km+m-1} 2\eta (f(x^t) - f(x^*))\Bigg]. \nonumber
\end{align}
Here, we used the fact that the system is synchronized after every epoch. The first step follows from telescopy sum and the definition of $\tilde{x}^k$. From Lemma 3 of \cite{Defazio14} (also see \cite{Johnson13}), we have
$$
\mathbb{E}[\|v^t\|^2] \leq  \\ 4L\mathbb{E}\left[f(x^{t}) - f(x^*) + f(\tilde{x}^k) - f(x^*) \right].
$$
Substituting this in the inequality above, we get the following bound:
\begin{align*}
&\left(2\eta - 8L\left( \frac{\eta^2 + L\Delta\tau^2\eta^3}{1 - 2L^2 \Delta\eta^2\tau^2}\right)\right) m \mathbb{E}[f(\tilde{x}^{k+1}) - f(x^*)] \\
 &\leq \left(\frac{2}{\lambda} + 8L\left( \frac{\eta^2+ L\Delta\tau^2\eta^3}{1 - 2L^2 \Delta\eta^2\tau^2}\right)m \right)\mathbb{E}[f(\tilde{x}^{k}) - f(x^*)].
\end{align*}
\end{proof}

\section*{Proof of Theorem~\ref{thm:t3}}
\begin{proof}
Let the present epoch be $k+1$. For simplicity, we assume that the iterates $x$ and $A$ used in the each iteration are from the same time step (index) i.e., $D(t) = D'(t)$ for all $t \in T$. Recall that $D(t)$ and $D'(t)$ denote the index used in the $t^{\text{th}}$ iteration of the algorithm. Our analysis can be extended to the case of $D(t) \neq D'(t)$ in a straightforward manner. We expand function $f$ as $f(x) = g(x) + h(x)$ where $g(x) = \frac{1}{n} \sum_{i \in S} f_i(x)$ and $h(x) = \frac{1}{n} \sum_{i \notin S} f_i(x)$. We define the following:
\begin{align*}
u^t &= \frac{1}{\eta} (x^{t+1} - x^t) = -\left[ \nabla f_{i_t}(x^{D(t)}) - \nabla f_{i_t}(\alpha_{i_t}^{D(t)}) + \frac{1}{n} \sum_i \nabla f_i(\alpha_i^{D(t)}) \right] \\
v^t &= -\left[ \nabla f_{i_t}(x^t) - \nabla f_{i_t}(\alpha_{i_t}^t) + \frac{1}{n} \sum_i \nabla f_i(\alpha_i^t) \right].
\end{align*}
We use the same Lyapunov function used in Theorem~\ref{thm:t1}. We recall the following definitions:
\begin{align*}
G_t &= \frac{1}{n} \sum_{i \in S} \left(f_{i}(\alpha_i^{t}) - f_i(x^*) - \langle \nabla f_i(x^*), \alpha_i^{t} - x^* \rangle \right) \\
R_t &= \mathbb{E}\left[ c\|x^{t} - x^*\|^2 + G_t \right].
\end{align*}
Using unbiasedness of the gradient we have $\mathbb{E}[u^t] = -\nabla f(x^{D(t)})$ and  $\mathbb{E}[v^t] = -\nabla f(x^t)$. Using this observation, we have the following:
\begin{align}
c \mathbb{E}[\|x^{t+1} - x^*\|^2] &= c \mathbb{E}[\|x^t + \eta u^t - x^*\|^2] \nonumber \\
&= c \mathbb{E}\left[ \|x^t - x^*\|^2\right] + c\eta^2 \underbrace{\mathbb{E}\left[\|u^t\|^2\right]}_{T_6} + 2c\eta \underbrace{\mathbb{E}\left[\langle x^t - x^*, u^t \rangle \right]}_{T_7}.
\end{align}
We bound term $T_6$ in the following manner:
\begin{align}
T_6 = \mathbb{E}\left[\|u^t\|^2\right] \leq 2\mathbb{E}\left[\|u^t - v^t\|^2\right] + 2 \mathbb{E}[\|v^t\|^2].
\end{align}
The first term can be bounded in the following manner:
\begin{align}
\mathbb{E}\left[\|u^t - v^t\|^2\right] &\leq \mathbb{E}\Big[\Big\|(\nabla f_{i_t}(x^{t}) - \nabla f_{i_t}(x^{D(t)})) -  (\nabla f_{i_t}(\alpha_{i_t}^{D(t)}) - \nabla f_{i_t}(\alpha_{i_t}^{t})) \nonumber \\
&\qquad \qquad + \frac{1}{n} \sum_i (\nabla f_{i}(\alpha_{i}^{t}) - \nabla f_{i}(\alpha_{i}^{D(t)})) \Big\|^2\Big] \nonumber \\
& \leq 3\mathbb{E}\left[\left\|\nabla f_{i_t}(x^{t}) - \nabla f_{i_t}(x^{D(t)})\right\|^2\right] + 3\mathbb{E}\left[\left\|\nabla f_{i_t}(\alpha_{i_t}^{D(t)}) - \nabla f_{i_t}(\alpha_{i_t}^{t})\right\|^2 \right] \nonumber \\
&\qquad \qquad + 3\mathbb{E}\left[\left\| \frac{1}{n} \sum_i (\nabla f_{i}(\alpha_{i}^{t}) - \nabla f_{i}(\alpha_{i}^{D(t)})) \right\|^2\right] \nonumber \\
& \leq 3\mathbb{E}\left[\left\|\nabla f_{i_t}(x^{t}) - \nabla f_{i_t}(x^{D(t)})\right\|^2\right] + 3\mathbb{E}\left[\left\|\nabla f_{i_t}(\alpha_{i_t}^{D(t)}) - \nabla f_{i_t}(\alpha_{i_t}^{t})\right\|^2 \right] \nonumber \\
&\qquad \qquad + \frac{3}{n}\sum_i\mathbb{E}\left[\left\| \nabla f_{i}(\alpha_{i}^{t}) - \nabla f_{i}(\alpha_{i}^{D(t)}) \right\|^2\right].
\label{eq:thm3-asyn-var-bound}
\end{align}
The second step follows from Lemma~\ref{lem:sq-lemma} for r = 3. The last step follows from simple application of Jensen's inequality. The first term can be bounded easily in the following manner:
\begin{align*}
\mathbb{E} \left[\|\nabla f_{i_t}(x^t) - \nabla f_{i_t}(x^{D(t)})\|^2\right] &\leq  L^2 \tau \sum_{d=D(t)}^{t-1} \mathbb{E} \left[\|x^{d+1} - x^{d}\|_{i_t}^2\right] \\
&\leq  L^2 \Delta\eta^2\tau \sum_{d=D(t)}^{t-1} \mathbb{E} \left[\|u^d\|^2\right].
\end{align*}

The second and third terms need more delicate analysis. The key insight for our analysis is that at most $\tau$ $\alpha_i$'s differ from time step $D(t)$ to $t$. This is due to the fact that the delay is bounded by $\tau$ and at most one $\alpha_i$ changes at each iteration. Furthermore, whenever there is a change in $\alpha_i$, it changes to one of the iterates $x^j$ for some $j = \{\max\{t-\tau,km\}, \dots, t\}$. With this intuition we bound the second term in the following fashion.
\begin{align*}
\mathbb{E}&\left[\left\|\nabla f_{i_t}(\alpha_{i_t}^{D(t)}) - \nabla f_{i_t}(\alpha_{i_t}^{t})\right\|^2 \right] \leq \frac{1}{n} \sum_{j = D(t)}^{t-1}\sum_{i \in S} \mathbb{E}\left[\mathbbm{1}(i = i_j)\left\|\nabla f_{i}(x^j) - \nabla f_{i}(\alpha_{i}^{D(t)})\right\|^2 \right] \\
&\leq \frac{2}{n} \sum_{j = D(t)}^{t-1}\sum_{i \in S} \mathbb{E}\left[\mathbbm{1}(i = i_j)\left(\left\|\nabla f_{i}(x^j) - \nabla f_{i}(x^*)\right\|^2 + \left\|\nabla f_{i}(\alpha_{i}^{D(t)}) - \nabla f_i(x^*)\right\|^2\right) \right] \\
&\leq \frac{2}{n^2} \sum_{j=D(t)}^{t-1} \sum_{i \in S} \mathbb{E}\left[\left\|\nabla f_{i}(x^j) - \nabla f_{i}(x^*)\right\|^2\right] + \frac{2}{n^2} \sum_{j=D(t)}^{t-1} \sum_{i \in S} \mathbb{E}\left[\left\|\nabla f_{i}(\alpha_{i}^{D(t)}) - \nabla f_i(x^*)\right\|^2 \right] \\
&\leq \frac{4L}{n}\sum_{j=D(t)}^{t-1} \mathbb{E}\left[\frac{1}{n}\sum_{i \in S} f_i(x^j) - f_i(x^*) - \langle \nabla f_i(x^*),x^j - x^* \rangle) \right] \\
& \qquad \qquad + \frac{4L\tau}{n} \mathbb{E}\left[\frac{1}{n}\sum_{i \in S} f_i(\alpha_i^{D(t)}) - f_i(x^*) - \langle \nabla f_i(x^*), \alpha_i^{D(t)} - x^* \rangle \right].
\end{align*}
The first inequality follows from the fact that if $\alpha_{i_t}^{D(t)}$ and $\alpha_{i_t}^t$ differ, then (a) $i_t$ should have been chosen in one of the iteration $j \in \{D(t),\dots,t-1\}$ and (b) $\alpha_{i_t}$ is changed to $x^j$ in that iteration. The second inequality follows from Lemma~\ref{lem:sq-lemma} for r = 2. The third inequality follows from the fact that the probability $P(i_j = i) = 1/n$. The last step directly follows from Lemma~\ref{lem:var-lemma}. Note that sum is over indices in $S$ since $\alpha_i$'s for $i \notin S$ do not change during the epoch.

The third term in Equation~\eqref{eq:thm3-asyn-var-bound} can be bounded by exactly the same technique we used for the second term. The bound, in fact, turns out to identical to second term since $i_t$ is chosen uniformly random. Combining all the terms we have
\begin{align*}
T_6 &\leq  2 \mathbb{E}[\|v^t\|^2] + 6L^2 \Delta\eta^2\tau \sum_{d=D(t)}^{t-1} \mathbb{E} \left[\|u^d\|^2\right] + \frac{48L}{n}\sum_{j=D(t)}^{t-1} \mathbb{E}\left[D_g(x^j,x^*) \right] + \frac{48L\tau}{n} \mathbb{E}\left[G_{D(t)} \right].
\end{align*} 

The term $T_7$ can be bounded in a manner similar to one in Theorem~\ref{thm:t2} to obtain the following (see proof of Theorem~\ref{thm:t2} for details):
\begin{align}
\mathbb{E}\langle x^t - x^*, u^t \rangle \leq \mathbb{E}\left[f(x^*) - f(x^t) + \frac{L\Delta\tau\eta^2}{2} \sum_{d = D(t)}^{t-1} \|u^d\|^2 \right].
\end{align}

We need the following bound for our analysis:
\begin{align*}
\sum_{j = 0}^{m-1} \left(1 - \frac{1}{\kappa}\right)^{m - 1 - j} \mathbb{E}[\|u^{km+j}\|^2] &\leq 2 \sum_{j = 0}^{m-1} \left(1 - \frac{1}{\kappa}\right)^{m - 1 - j} \mathbb{E}[\|v^{km+j}\|^2] \\
& \qquad + \sum_{t=km}^{km+m-1} 6L^2 \Delta\eta^2\tau \sum_{d=D(t)}^{t-1} \mathbb{E} \left[\|u^d\|^2\right] \\
& \qquad + \sum_{t=km}^{km+m-1} \frac{48L}{n}\sum_{j=D(t)}^{t-1} \mathbb{E}\left[D_g(x^j,x^*) \right] \\
& \qquad + \sum_{t=km}^{km+m-1} \frac{48L\tau}{n} \mathbb{E}\left[G_{D(t)} \right].
\end{align*}
The above inequality follows directly from the bound on $T_6$ by adding over all $t$ in the epoch. Under the condition
$$
\eta^2 \leq \left(1 - \frac{1}{\kappa}\right)^{m-1}\frac{1}{12L^2\Delta\tau^2}.
$$
we have the following inequality
\begin{align}
\sum_{j = 0}^{m-1} \left(1 - \frac{1}{\kappa}\right)^{m - 1 - j} \mathbb{E}[\|u^{km+j}\|^2] &\leq 4 \sum_{j = 0}^{m-1} \left(1 - \frac{1}{\kappa}\right)^{m - 1 - j} \mathbb{E}[\|v^{km+j}\|^2] \nonumber \\
& \qquad + \sum_{t=km}^{km+m-1} \frac{96L}{n}\sum_{j=D(t)}^{t-1} \mathbb{E}\left[D_g(x^j,x^*) \right] \nonumber \\
& \qquad + \sum_{t=km}^{km+m-1} \frac{96L\tau}{n} \mathbb{E}\left[G_{D(t)} \right].
\label{eq:ut-bound}
\end{align}
The above inequality follows from the fact that 
\begin{align*}
\sum_{t=km}^{km+m-1} 6L^2 \Delta\eta^2\tau \sum_{d=D(t)}^{t-1} \mathbb{E} \left[\|u^d\|^2\right] &\leq \sum_{t=km}^{km+m-1} 6L^2 \Delta\eta^2\tau^2 \mathbb{E} \left[\|u^t\|^2\right] \\
&\leq \frac{1}{2}\sum_{j = 0}^{m-1} \left(1 - \frac{1}{\kappa}\right)^{m - 1 - j} \mathbb{E}[\|u^{km+j}\|^2].
\end{align*}
The above relationship is due to the condition on $\eta$ and the fact that any $d \in \{D(t),\dots,t-1\}$ for at most $\tau$ values of $t$. We have the following:
\begin{align}
R_{t+1} &= c \mathbb{E}\left[ \|x^t - x^*\|^2\right] + c\eta^2 \mathbb{E}\left[\|u^t\|^2\right] + 2c\eta \mathbb{E}\left[\langle x^t - x^*, u^t \rangle \right] + \mathbb{E}\left[G_{t+1} \right] \nonumber \\
&:= \left(1- \frac{1}{\kappa}\right) R_{t} + e_t. \label{eq:thm3-rec}
\end{align}
We bound $e_t$ in the following manner:
\begin{align*}
e_t &= \frac{c}{\kappa} \|x^t - x^*\|^2 + \left(\frac{1}{\kappa} - \frac{1}{n} \right)\mathbb{E}[G_t] + c\eta^2 \mathbb{E}\left[\|u^t\|^2\right] + 2c\eta \mathbb{E}\left[\langle x^t - x^*, u^t \rangle \right] + \mathbb{E}\left[G_{t+1}\right] \\
&= \frac{c}{\kappa} \|x^t - x^*\|^2 + \left(\frac{1}{\kappa} - \frac{1}{n} \right)\mathbb{E}[G_t] + c\eta^2 \mathbb{E}\left[\|u^t\|^2\right] + 2c\eta \mathbb{E}\left[\langle x^t - x^*, u^t \rangle \right] + \frac{1}{n}\mathbb{E}[D_g(x^t, x^*)] \\
&\leq - \left(2c\eta - \frac{2c}{\kappa\lambda}\right) \mathbb{E}\left[f(x^t) - f(x^*)\right] + \left(\frac{1}{\kappa} - \frac{1}{n} \right)\mathbb{E}[G_t] + c\eta^2 \mathbb{E}[\|u^t\|^2] \\
& \qquad \qquad + cL\Delta\tau\eta^3 \sum_{d = D(t)}^{t-1} \mathbb{E}\left[\|u^d\|^2 \right] + \frac{1}{n}\mathbb{E}[D_g(x^t, x^*)].
\end{align*}

The second equality follows from the definition of $G_{t+1}$ (see Equation~\eqref{eq:thm1-gt-rec}).
\begin{align*}
\mathbb{E}[G_{t+1}] &= \left(1 - \frac{1}{n}\right) \mathbb{E} \left[G_t \right] + \frac{1}{n}\mathbb{E}[D_g(x^t, x^*)].
\end{align*}

Applying the recurrence relationship in Equation~\eqref{eq:thm3-rec} with the derived bound on $e_t$, we have 
\begin{align*}
R_{km+m} &\leq \left(1 - \frac{1}{\kappa}\right)^m \tilde{R}_{k} + \sum_{j=0}^{m-1}  \left(1 - \frac{1}{\kappa}\right)^{m-1 - j} e_{km+j} \\
&\leq \left(1 - \frac{1}{\kappa}\right)^m \tilde{R}_{k} + \sum_{j=0}^{m-1} \left(1 - \frac{1}{\kappa}\right)^{m-1 - j} e'_{km+j},
\end{align*}
where $e'_{t}$ is defined as follows
\begin{align*}
\tilde{R}_{k} &= \mathbb{E}\left[ c\|\tilde{x}^{k} - x^*\|^2 + \tilde{G}_k \right] \\
e'_{t} &= - \left(2c\eta - \frac{2c}{\kappa\lambda}\right) \mathbb{E}\left[f(x^t) - f(x^*)\right] + \left(\frac{1}{\kappa} - \frac{1}{n} \right)\mathbb{E}[G_t] \\
& \qquad \qquad + \left(c\eta^2 + \left(1 - \frac{1}{\kappa} \right)^{-\tau} cL\Delta\tau^2\eta^3 \right) \mathbb{E}[\|u^t\|^2]
+ \frac{1}{n}\mathbb{E}[D_g(x^t, x^*)].
\end{align*}
The last inequality follows from that fact that the delay is at most $\tau$. In particular, each index $j \in \{D(t), \dots, t-1\}$ occurs at most $\tau$ times. We use the following notation for ease of exposition:
\begin{align*}
\zeta = \left(c\eta^2 + \left(1 - \frac{1}{\kappa} \right)^{-\tau} cL\Delta\tau^2\eta^3 \right).
\end{align*}
Substituting the bound in Equation~\eqref{eq:ut-bound}, we get the following:
\begin{align}
R_{km+m} &\leq  \left(1 - \frac{1}{\kappa}\right)^m \tilde{R}_{k} - \left(2c\eta - \frac{2c}{\kappa\lambda}\right) \sum_{j=0}^{m-1} \left(1 - \frac{1}{\kappa}\right)^{m - 1 - j} \mathbb{E}\left[f(x^{km+j}) - f(x^*)\right] \nonumber \\
& \qquad + 4\zeta \sum_{j = 0}^{m-1} \left(1 - \frac{1}{\kappa}\right)^{m - 1 - j} \mathbb{E}[\|v^{km+j}\|^2] \nonumber \\
& \qquad + \left[\frac{96\zeta L\tau}{n}  \left(1 - \frac{1}{\kappa}\right)^{-\tau} + \frac{1}{n} \right] \ \sum_{j=0}^{m-1} \left(1 - \frac{1}{\kappa}\right)^{m - 1 - j} \mathbb{E}\left[D_g(x^{km+j},x^*) \right] \nonumber \\
& \qquad + \left[\frac{1}{\kappa} + \frac{96\zeta L\tau}{n}  \left(1 - \frac{1}{\kappa}\right)^{-\tau} - \frac{1}{n} \right] \ \sum_{j=0}^{m-1} \left(1 - \frac{1}{\kappa}\right)^{m - 1 - j} \mathbb{E}\left[G_{D(km+j)} \right].
\label{eq:thm3-par-rec}
\end{align}

We now use the following previously used bound on $v^t$ (see bound $T_2$ in the proof of Theorem~\ref{thm:t1}):
\begin{align*}
 \mathbb{E}[\|v^{t}\|^2] \leq 2L\left(1 + \frac{1}{\beta}\right)\left[G_{t} + D_h(\tilde{x}^k,x^*)\right] + 2L(1+\beta) \mathbb{E}[f(x^t) - f(x^*)].
\end{align*}
Substituting the above bound on $v^t$ in Equation~\eqref{eq:thm3-par-rec}, we get the following:
\begin{align}
R_{km+m} &\leq  \left(1 - \frac{1}{\kappa}\right)^m \tilde{R}_{k} \nonumber \\
& \qquad \qquad - \left[2c\eta - 8\zeta L(1+\beta) - \frac{2c}{\kappa\lambda} - \frac{96\zeta L\tau}{n}  \left(1 - \frac{1}{\kappa}\right)^{-\tau} - \frac{1}{n} \right] \times \nonumber \\
&  \qquad \qquad \qquad \qquad \qquad \sum_{j=0}^{m-1} \left(1 - \frac{1}{\kappa}\right)^{m - 1 - j} \mathbb{E}\left[f(x^{km+j}) - f(x^*)\right] \nonumber \\
& \qquad \qquad + \left[\frac{1}{\kappa} + 8\zeta L\left(1 + \frac{1}{\beta}\right) + \frac{96\zeta L\tau}{n}  \left(1 - \frac{1}{\kappa}\right)^{-\tau} - \frac{1}{n} \right] \times \nonumber \\
& \qquad \qquad \qquad \qquad \qquad \sum_{j=0}^{m-1} \left(1 - \frac{1}{\kappa}\right)^{m - 1 - j} \mathbb{E}\left[G_{km+j} \right] \nonumber \\
& \qquad \qquad  + 8\zeta L\left(1 + \frac{1}{\beta}\right)\sum_{j=0}^{m-1} \left(1 - \frac{1}{\kappa}\right)^{m - 1 - j} \mathbb{E}\left[D_h(\tilde{x}^k,x^*)\right]  \nonumber \\
& \leq   \frac{2c}{\lambda}\left(1 - \frac{1}{\kappa}\right)^m \mathbb{E}\left[f(\tilde{x}^k) - f(x^*)\right] + \left(1 - \frac{1}{\kappa}\right)^m \mathbb{E}\left[ \tilde{G}_k \right] \nonumber \\
& \qquad \qquad - \left[2c\eta - 8\zeta L(1+\beta) - \frac{2c}{\kappa\lambda} - \frac{96\zeta L\tau}{n}  \left(1 - \frac{1}{\kappa}\right)^{-\tau} - \frac{1}{n} \right] \times \nonumber \\
& \qquad \qquad \qquad \qquad \sum_{j=0}^{m-1} \left(1 - \frac{1}{\kappa}\right)^{m - 1 - j} \mathbb{E}\left[f(x^{km+j}) - f(x^*)\right] \nonumber \\
& \qquad \qquad +  8\zeta L\left(1 + \frac{1}{\beta}\right)\kappa \left[1 - \left(1 - \frac{1}{\kappa}\right)^m \right] \mathbb{E}\left[D_h(\tilde{x}^k,x^*)\right].
\label{eq:thm3-final-rec}
\end{align}
The first step is due to the Bregman divergence based inequality $D_f(x,x^*) \geq D_g(x,x^*)$. The second step follows from the expanding $\tilde{R}_k$ and using the strong convexity of function $f$.
For brevity, we use the following notation:
\begin{align*}
\gamma_a &= \kappa \left[1 - \left(1 - \frac{1}{\kappa}\right)^m \right] \left[2c\eta - 8\zeta L(1+\beta) - \frac{2c}{\kappa\lambda} - \frac{96\zeta L\tau}{n}  \left(1 - \frac{1}{\kappa}\right)^{-\tau} - \frac{1}{n} \right] \\
\theta_a &= \max \left\lbrace \left[\frac{2c}{\gamma_a\lambda}\left(1 - \frac{1}{\kappa}\right)^m + \frac{8\zeta L\left(1 + \frac{1}{\beta}\right)}{\gamma_a}  \kappa \left[1 - \left(1 - \frac{1}{\kappa}\right)^m \right]  \right], \left(1 - \frac{1}{\kappa}\right)^m  \right\rbrace .
\end{align*}
We now use the fact that $\tilde{x}^{k+1}$ is chosen randomly from $\{x^{km}, \dots, x^{km+m-1}\}$ with probabilities proportional to $\{(1-1/\kappa)^{m-1}, \dots, 1\}$. Hence, we have the following inequality from Equation~\eqref{eq:thm3-final-rec}:
\begin{align*}
& \mathbb{E}\left[f(\tilde{x}^{k+1}) - f(x^*) + \frac{1}{\gamma_a} \tilde{G}_{k+1}\right] \leq  \theta_a \ \mathbb{E}\left[f(\tilde{x}^{k}) - f(x^*) + \frac{1}{\gamma_a} \tilde{G}_{k}\right],
\end{align*}
where $\theta_a < 1$ is a constant that depends on the parameters used in the algorithm.
\end{proof}

\subsection*{Remarks about the parameters in Theorem~\ref{thm:t1} \& Theorem~\ref{thm:t3}}

In this section, we briefly remark about the parameters in Theorems~\ref{thm:t1} \& ~\ref{thm:t3}. For Theorem~\ref{thm:t1}, suppose we use the following instantiation of the parameters:
\begin{align*}
\eta &= \frac{1}{16(\lambda n + L)} \\
\kappa &= \frac{4}{\lambda \eta} = 64\left(n + \frac{L}{\lambda}\right) \\
\beta &= \frac{2\lambda n + L}{L} \\
c &= \frac{2}{\eta n} = 32\left(\lambda + \frac{L}{n}\right).
\end{align*}
Then we have,
\begin{align*}
\theta = \max \left\lbrace \left[\frac{2(1 - \frac{1}{\kappa})^{m}}{3\left(1 - (1 - \frac{1}{\kappa})^{m}\right)} + \frac{1}{3\left(1 + \frac{2\lambda n}{L}\right)}\right], \left(1 - \frac{1}{\kappa}\right)^{m} \right\rbrace.
\end{align*}
In the interesting case of $L/\lambda = n$ (high condition number regime), since $\kappa = \Theta(n)$, one can obtain a constant $\theta$ (say $\theta = 0.5$) with $m = O(n)$. This leads to $\epsilon$ accuracy in the objective function after $O( \log(1/\epsilon))$ epochs of \svag. When $m = O(n)$, the computational complexity of each epoch of \svag \ is $O(n)$. Hence, the total computational complexity of \svag \ is $O(n \log(1/\epsilon))$. On the other hand, because $L/\lambda = n$, batch gradient descent method requires $O(n \log(1/\epsilon))$ iterations to achieve $\epsilon$ accuracy in the objective value. Since the complexity of each iteration of gradient descent is $O(n)$ (as it passes through the whole dataset for calculating the gradient), the overall computational complexity of batch gradient descent is $O(n^2 \log(1/\epsilon))$. In general, for high condition number regimes (which is typically the case in machine learning applications), \svag \ (like \svrg, \ \saga) will be significantly faster than the batch gradient methods. Furthermore, the convergence rate is strictly better than the sublinear rate obtained for \sgd.

The parameter instantiations for Theorem~\ref{thm:t3} are much more involved. Suppose $\Delta^{1/2} \tau < 1$ (this is the sparse regime that is typically of interest to the machine learning community) and $m > n > 9\tau$. The other case ( $\Delta^{1/2} \tau \geq 1$) can be analyzed in a similar fashion. We set the following parameters:
\begin{align*}
\eta &= \frac{\left(1 - \frac{1}{\kappa}\right)^m}{64(\lambda n + L)} \\
\kappa &= \frac{4}{\lambda \eta} = \frac{256\left(n + \frac{L}{\lambda}\right)}{(1 - \frac{1}{\kappa})^m} \\
\beta &= \frac{2\lambda n + L}{L} \\
c &= \frac{2}{\eta n} = 32\left(\lambda + \frac{L}{n}\right).
\end{align*}
Then we have the following:
\begin{align*}
\zeta &\leq \frac{\left(1 - \frac{1}{\kappa}\right)^m}{32n(\lambda n + L)} \left(1 + \frac{L}{64(\lambda n + L)}\right) \\
\theta_a &\leq \max \left\lbrace \left[\frac{6(1 - \frac{1}{\kappa})^{m}}{7\left(1 - (1 - \frac{1}{\kappa})^{m}\right)} + \frac{195\left(1 - \frac{1}{\kappa}\right)^\tau}{448(1 + \frac{2\lambda n}{L})} \right], \left(1 - \frac{1}{\kappa}\right)^m\right\rbrace
\end{align*}

Again, in the case of $L/\lambda = n$, we can obtain constant $\theta_a$ (say $\theta_a = 0.5$) with $m = \Theta(n)$ and $\kappa = \Theta(n)$. The constants in the parameters are not optimized and can be improved by a more careful analysis. Furthermore, sharper constants can be obtained in specific cases. For example, see \cite{Johnson13} and Theorem~\ref{thm:t2} for synchronous and asynchronous convergence rates of \svrg \ respectively. Similarly, sharper constants for \saga \ can also be derived by simple modifications of the analysis.
 
\section*{Other Lemmatta}
\begin{lemma}
\label{lem:var-lemma}
\cite{Johnson13} For any $\alpha_i \in \mathbb{R}^d$ where $i \in [n]$ and $x^*$, we have
\begin{align*}
\mathbb{E}\left[\|\nabla f_{i_t}(\alpha_{i_t}) - \nabla f_{i_t}(x^*)\|^2 \right] &\leq \frac{2L}{n} \sum_i \left[f_i(\alpha_i) - f(x^*) - \left\langle \nabla f_i(x^*), \alpha_i - x^* \right\rangle \right].
\end{align*}
\end{lemma}

\begin{lemma}
\label{lem:breg}
Suppose $f:\mathbb{R}^d \rightarrow \mathbb{R}$ and $f = g + h$ where $f,g$ and $h$ are convex and differentiable. $x^*$ is the optimal solution to $\arg\min_x f(x)$ then we have the following
\begin{align*}
D_f(x,x^*) &= f(x) - f(x^*) \\
D_f(x,x^*) &= D_g(x,x^*) + D_h(x,x^*) \\
D_f(x,x^*) &\geq D_g(x,x^*).
\end{align*}
\end{lemma}
\begin{proof}
The proof follows trivially from the fact that $x^*$ is the optimal solution and linearity and non-negative properties of Bregman divergence.
\end{proof}

\begin{lemma}
\label{lem:sq-lemma}
For random variables $z_1, \dots, z_r$, we have
\begin{align*}
\mathbb{E}\left[ \|z_1 + ... + z_r\|^2 \right] \leq r \mathbb{E}\left[\|z_1\|^2 + ... + \|z_r\|^2\right].
\end{align*}
\end{lemma}

\end{document}